\documentclass[11pt,letterpaper]{article}
\usepackage[margin=1in]{geometry}
\usepackage{times}
\usepackage{natbib}
\setcitestyle{authoryear,round,citesep={;},aysep={,},yysep={;}}

\usepackage{etoolbox}
\newtoggle{arxiv}
\toggletrue{arxiv}

\usepackage{amsmath,amsfonts,bm}

\def\eqref#1{equation~\ref{#1}}

\def\1{\bm{1}}

\DeclareMathAlphabet{\mathsfit}{\encodingdefault}{\sfdefault}{m}{sl}
\SetMathAlphabet{\mathsfit}{bold}{\encodingdefault}{\sfdefault}{bx}{n}

\def\eqref#1{Equation~\textup{(\ref{#1})}}

\usepackage{hyperref}
\usepackage{cleveref}
\usepackage{url}
\usepackage{graphicx}
\usepackage{tikz}
\usepackage{amsmath,amssymb}
\usepackage{amsthm}

\newtheorem{proposition}{Proposition}
\newtheorem{corollary}{Corollary}

\crefname{proposition}{Proposition}{Propositions}
\crefname{theorem}{Theorem}{Theorems}
\crefname{corollary}{Corollary}{Corollaries}
\usetikzlibrary{arrows.meta,calc,positioning}
\usepackage{float}
\usepackage{booktabs,graphicx}
\usepackage{tabularx}

\usepackage{booktabs,tabularx,array,float}
\usepackage[most]{tcolorbox}
\usepackage{autonum}
\usepackage{nicefrac}
\tcbuselibrary{skins}

\definecolor{promptblue}{HTML}{4A8993}
\definecolor{promptbackground}{HTML}{F3F9FA}
\newcolumntype{Y}{>{\raggedright\arraybackslash}X}
\newcommand{\chatline}[1]{{\ttfamily\footnotesize\detokenize{#1}\par}}
\tcbset{promptbox/.style={enhanced,colback=promptbackground,colframe=promptblue,coltitle=white,fonttitle=\bfseries,boxrule=0.7pt,arc=2mm,left=2mm,right=2mm,top=1.5mm,bottom=1.5mm,before skip=6pt,after skip=6pt}}

\usepackage{amsthm}
\usepackage{subcaption}

\definecolor{linkblue}{RGB}{0,51,128}
\hypersetup{
  colorlinks = true,
  linkcolor  = red!65!black,
  citecolor  = blue!75!black,
  urlcolor   = blue!75!black,
  filecolor  = blue!75!black
}
\usepackage[ruled,vlined]{algorithm2e}
\usepackage{xcolor}
\usepackage{tablefootnote}

\crefname{algocf}{Algorithm}{Algorithms}
\crefname{equation}{Equation}{Equations}
\creflabelformat{equation}{(#2#1#3)}

\definecolor{darkbrown}{rgb}{0.4, 0.2, 0.0}
\definecolor{commentblue}{rgb}{0.0, 0.0, 0.8}

\usepackage[most]{tcolorbox}
\usepackage[suppress]{color-edits}
\addauthor{ab}{red}
\addauthor{nd}{teal}
\addauthor{ng}{blue}

\renewcommand{\epsilon}{\varepsilon}

\def\ddefloop#1{\ifx\ddefloop#1\else\ddef{#1}\expandafter\ddefloop\fi}
\def\ddef#1{\expandafter\def\csname bb#1\endcsname{\ensuremath{\mathbb{#1}}}}
\ddefloop ABCDEFGHIJKLMNOPQRSTUVWXYZ\ddefloop

\def\ddefloop#1{\ifx\ddefloop#1\else\ddef{#1}\expandafter\ddefloop\fi}
\def\ddef#1{\expandafter\def\csname frak#1\endcsname{\ensuremath{\mathfrak{#1}}}}
\ddefloop ABCDEFGHIJKLMNOPQRSTUVWXYZ\ddefloop

\def\ddefloop#1{\ifx\ddefloop#1\else\ddef{#1}\expandafter\ddefloop\fi}
\def\ddef#1{\expandafter\def\csname fr#1\endcsname{\ensuremath{\mathfrak{#1}}}}
\ddefloop ABCDEFGHIJKLMNOPQRSTUVWXYZ\ddefloop

\def\ddefloop#1{\ifx\ddefloop#1\else\ddef{#1}\expandafter\ddefloop\fi}
\def\ddef#1{\expandafter\def\csname eul#1\endcsname{\ensuremath{\EuScript{#1}}}}
\ddefloop ABCDEFGHIJKLMNOPQRSTUVWXYZ\ddefloop

\def\ddefloop#1{\ifx\ddefloop#1\else\ddef{#1}\expandafter\ddefloop\fi}
\def\ddef#1{\expandafter\def\csname scr#1\endcsname{\ensuremath{\mathscr{#1}}}}
\ddefloop ABCDEFGHIJKLMNOPQRSTUVWXYZ\ddefloop

\def\ddefloop#1{\ifx\ddefloop#1\else\ddef{#1}\expandafter\ddefloop\fi}
\def\ddef#1{\expandafter\def\csname b#1\endcsname{\ensuremath{\mathbf{#1}}}}
\ddefloop ABCDEGHIJKLMNOPQRSTUVWXYZ\ddefloop

\def\ddefloop#1{\ifx\ddefloop#1\else\ddef{#1}\expandafter\ddefloop\fi}
\def\ddef#1{\expandafter\def\csname bhat#1\endcsname{\ensuremath{\hat{\mathbf{#1}}}}}
\ddefloop ABCDEFGHIJKLMNOPQRSTUVWXYZ\ddefloop

\def\ddefloop#1{\ifx\ddefloop#1\else\ddef{#1}\expandafter\ddefloop\fi}
\def\ddef#1{\expandafter\def\csname btil#1\endcsname{\ensuremath{\tilde{\mathbf{#1}}}}}
\ddefloop ABCDEFGHIJKLMNOPQRSTUVWXYZ\ddefloop

\def\ddefloop#1{\ifx\ddefloop#1\else\ddef{#1}\expandafter\ddefloop\fi}
\def\ddef#1{\expandafter\def\csname bst#1\endcsname{\ensuremath{\mathbf{#1}^\star}}}
\ddefloop ABCDEFGHIJKLMNOPQRSTUVWXYZ\ddefloop

\def\ddefloop#1{\ifx\ddefloop#1\else\ddef{#1}\expandafter\ddefloop\fi}
\def\ddef#1{\expandafter\def\csname bst#1\endcsname{\ensuremath{\mathbf{#1}^\star}}}
\ddefloop abcdeghijklmnopqrstuvwxyz\ddefloop

\def\ddefloop#1{\ifx\ddefloop#1\else\ddef{#1}\expandafter\ddefloop\fi}
\def\ddef#1{\expandafter\def\csname bhat#1\endcsname{\ensuremath{\hat{\mathbf{#1}}}}}
\ddefloop abcdefghijklmnopqrstuvwxyz\ddefloop

\def\ddefloop#1{\ifx\ddefloop#1\else\ddef{#1}\expandafter\ddefloop\fi}
\def\ddef#1{\expandafter\def\csname b#1\endcsname{\ensuremath{\mathbf{#1}}}}
\ddefloop abcdeghijklnopqrstuvwxyz\ddefloop

\def\ddefloop#1{\ifx\ddefloop#1\else\ddef{#1}\expandafter\ddefloop\fi}
\def\ddef#1{\expandafter\def\csname barb#1\endcsname{\ensuremath{\bar{\mathbf{#1}}}}}
\ddefloop abcdefghijklmnopqrstuvwxyz\ddefloop

\def\ddef#1{\expandafter\def\csname c#1\endcsname{\ensuremath{\mathcal{#1}}}}
\ddefloop ABCDEFGHIJKLMNOPQRSTUVWXYZ\ddefloop
\def\ddef#1{\expandafter\def\csname h#1\endcsname{\ensuremath{\widehat{#1}}}}
\ddefloop ABCDEFGHIJKLMNOPQRSTUVWXYZ\ddefloop
\def\ddef#1{\expandafter\def\csname hc#1\endcsname{\ensuremath{\widehat{\mathcal{#1}}}}}
\ddefloop ABCDEFGHIJKLMNOPQRSTUVWXYZ\ddefloop
\def\ddef#1{\expandafter\def\csname t#1\endcsname{\ensuremath{\widetilde{#1}}}}
\ddefloop ABCDEFGHIJKLMNOPQRSTUVWXYZ\ddefloop
\def\ddef#1{\expandafter\def\csname tc#1\endcsname{\ensuremath{\widetilde{\mathcal{#1}}}}}
\ddefloop ABCDEFGHIJKLMNOPQRSTUVWXYZ\ddefloop

\newcommand{\norm}[1]{\left\|#1\right\|}

\newcommand{\ee}{\mathbb{E}}
\newcommand{\pp}{\mathbb{P}}

\newcommand{\lreapo}{\mathcal{L}_{\mathrm{ReAPO}}}

\title{ReAPO: Retry-Aware Policy Optimization for Efficient Exploration in LLM Reasoning}

\author{\small Nghi Dao$^{1}$\thanks{\href{mailto:nmd2167@columbia.edu}{\texttt{nmd2167@columbia.edu}}} \quad Calder Morton-Ferguson$^{2}$\thanks{\href{mailto:calder.morton-ferguson@yale.edu}{\texttt{calder.morton-ferguson@yale.edu}}} \quad Noah Golowich$^{3}$\thanks{\href{mailto:nzg@cs.utexas.edu}{\texttt{nzg@cs.utexas.edu}}}\quad Adam Block$^{1}$\thanks{\href{mailto:adam.block@columbia.edu}{\texttt{adam.block@columbia.edu}}} \\
\small $^{1}$Columbia University \quad $^{2}$Yale University \quad $^{3}$UT Austin}
\date{}

\begin{document}

\maketitle

\begin{abstract}
RLVR improves language-model reasoning, but commonly relies on independent stochastic rollouts for exploration, which can repeatedly sample similar failed strategies. Prior work has seen success with approaches that guide exploration by introducing dependence into the sampling process by including incorrect solutions in the context for subsequent attempts. Still, these methods introduce the additional question of how to assign credit to each attempt.  We propose to view the credit assignment through the lens of \emph{Stochastic Shortest Paths}, a framework in classical reinforcement learning that seeks to minimize the expected cumulative cost to reach a terminal state (in our case, a verified successful response).  We show that, under suitable simplifying assumptions, this perspective naturally leads to a simple modification to GRPO, which we call \textbf{Retry-Aware Policy Optimization (ReAPO)}.  ReAPO allocates additional generation to failed trajectories, allowing successful retries to be reinforced without rewarding preceding failures, and assigns credit in a theoretically principled way to each attempt.
Across Qwen3 and OLMo-2 models from 0.6B to 4B and eight mathematical and scientific reasoning benchmarks, ReAPO improves first-attempt performance over GRPO as well as natural recent baseline algorithms, while additional ablations demonstrate the utility of both serial sampling and ReAPO's credit assignment. On Qwen3-0.6B, ReAPO also outperforms GRPO with twice the group size or twice the training steps, reaching their peak accuracy with 39\% and 31\% fewer generated tokens, respectively.  Taken together, our results suggest that ReAPO is a simple yet effective approach for improving LM reasoning performance through improved exploration and retry-aware credit assignment.\footnote{All code to reproduce our results can be found at \url{https://github.com/Nghi-Dao/ReAPO}.}

\end{abstract}

\section{Introduction}
\label{sec:introduction}

Reinforcement learning with verifiable rewards (RLVR) has emerged as an effective approach for improving the reasoning capabilities of large language models (LLMs) without requiring human-written reasoning traces
\citep{shao2024deepseekmathpushinglimitsmathematical,Guo_2025,wen2025reinforcementlearningverifiablerewards,lambert2025tulu3pushingfrontiers,hu2025openreasonerzeroopensourceapproach,zeng2025simplerlzooinvestigatingtamingzero}.
Given a problem with an easily verifiable answer (e.g., formalized mathematics or compilable code), RLVR samples multiple responses from a policy, evaluates each response, and reinforces responses that receive higher rewards. Most contemporary approaches to RLVR are based on the seminal algorithm Group Relative Policy Optimization (GRPO), which estimates the quality of individual responses by comparing verified rewards within the group of sampled responses.
Despite the success of GRPO, the algorithm is limited by its reliance on \emph{policy stochasticity} to explore semantically distinct reasoning trajectories.  Indeed, like many policy gradient methods, GRPO does not explicitly encourage exploration, the process by which a model actively seeks out novel reasoning strategies; moreover, this deficiency has been noted in prior work that questions the ability of RL methods to drive genuine reasoning improvements beyond the capabilities of the base model \citep{yue2025doesreinforcementlearningreally,karan2025reasoningsamplingbasemodel,huang2025bestofn,chen2026coverage,wen2025reinforcementlearningverifiablerewards,huang2025self}.

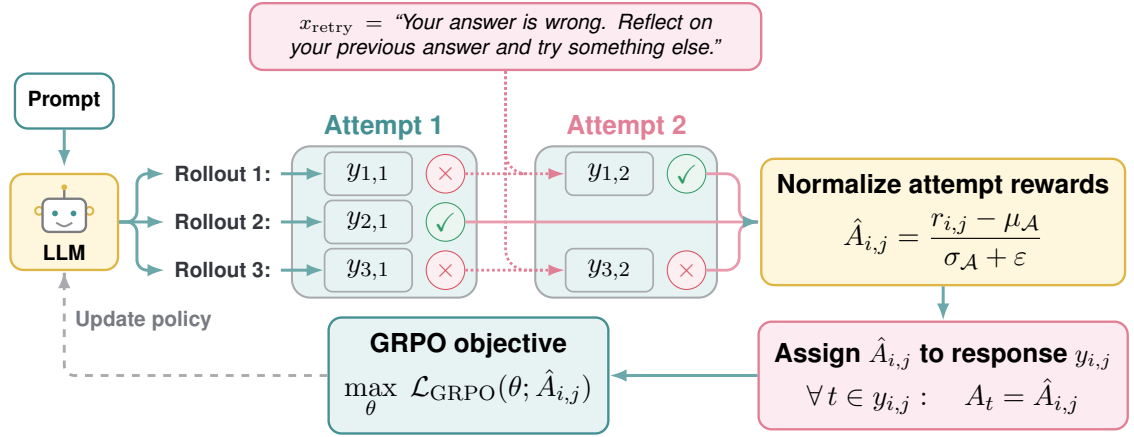
\begin{figure}[t]
    \centering
    \resizebox{\textwidth}{!}{%
        \newcommand{\RetryPromptText}{%
  $x_{\mathrm{retry}}={}$\textit{``Your answer is wrong. Reflect on your previous answer and try something else.''}%
}

\definecolor{ink}{HTML}{33363C}
\definecolor{muted}{HTML}{777A80}
\definecolor{teal}{HTML}{438F91}
\definecolor{tealLight}{HTML}{E7F3F3}
\definecolor{orange}{HTML}{DF6634}
\definecolor{orangeLight}{HTML}{FFF2EA}
\definecolor{pink}{HTML}{E17F96}
\definecolor{pinkLight}{HTML}{FCEBF0}
\definecolor{gold}{HTML}{DDB646}
\definecolor{goldLight}{HTML}{FFF7DD}
\definecolor{grayLine}{HTML}{A8AAAE}
\definecolor{grayLight}{HTML}{F3F5F6}
\definecolor{green}{HTML}{319A68}
\definecolor{red}{HTML}{D85061}

\tikzset{
  font=\sffamily,
  module/.style={rounded corners=5pt, fill=white, draw=#1, line width=.9pt,
                 minimum height=9mm, inner xsep=8pt, align=center},
  attemptframe/.style={rounded corners=7pt, fill=grayLight, draw=grayLine,
                       line width=.8pt},
  attempt/.style={rounded corners=3pt, fill=tealLight, draw=grayLine,
                  line width=.65pt,
                  minimum width=1.18cm, minimum height=5.8mm, align=center,
                  inner sep=3pt},
  normalizebox/.style={rounded corners=5pt, fill=goldLight, draw=gold,
                       line width=.9pt, minimum width=3.85cm,
                       minimum height=1.18cm, align=center, inner sep=6pt},
  advantagebox/.style={rounded corners=5pt, fill=pinkLight, draw=pink,
                       line width=.9pt, minimum width=4.15cm,
                       minimum height=1.05cm, align=center, inner sep=6pt},
  grpobox/.style={rounded corners=5pt, fill=tealLight, draw=teal,
                  line width=.9pt, minimum width=3.25cm,
                  minimum height=1.05cm, align=center, inner sep=6pt},
  flow/.style={-{Latex[length=2.3mm,width=1.7mm]}, draw=teal!78,
               line width=1.15pt, rounded corners=4pt},
  collect/.style={draw=pink!78, line width=1.08pt, rounded corners=4pt},
  retry/.style={-{Latex[length=2.2mm,width=1.6mm]}, draw=pink,
                line width=1.05pt, densely dotted, rounded corners=4pt},
  retryline/.style={draw=pink, line width=1.05pt,
                    densely dotted, rounded corners=4pt},
  update/.style={-{Latex[length=2.3mm,width=1.7mm]}, draw=grayLine,
                 line width=1.1pt, dashed, rounded corners=4pt},
  result/.style={circle, minimum size=5.2mm, inner sep=0pt,
                 font=\sffamily\bfseries\small, draw=#1, fill=#1!9, text=#1},
  smalllabel/.style={font=\sffamily\scriptsize, text=muted, align=center},
}

\newcommand{\robothead}{%
  \begin{tikzpicture}[baseline=-.55ex,scale=.43]
    \draw[ink!65,line width=.9pt] (0,.70)--(0,.98);
    \fill[gold] (0,1.06) circle (.09);
    \filldraw[fill=goldLight,draw=gold,line width=.8pt] (-.79,.32) circle (.12);
    \filldraw[fill=goldLight,draw=gold,line width=.8pt] (.79,.32) circle (.12);
    \filldraw[fill=white,draw=ink!65,line width=.9pt,rounded corners=2.5pt]
      (-.68,-.34) rectangle (.68,.70);
    \fill[teal] (-.29,.24) circle (.09);
    \fill[teal] (.29,.24) circle (.09);
    \draw[ink!55,line width=.8pt] (-.25,-.05) .. controls (0,-.20) .. (.25,-.05);
  \end{tikzpicture}%
}

\begin{tikzpicture}[x=1cm,y=1cm]
  \node[module=teal,minimum width=1.22cm,minimum height=7.4mm,
        inner xsep=5pt,inner ysep=3pt] (prompt) at (2.20,5.18)
    {\scriptsize\textbf{Prompt}};
  \node[module=gold,fill=goldLight,minimum width=1.45cm,inner ysep=4pt]
    (llm) at (2.20,3.55) {\robothead\\[-2pt]\scriptsize\textbf{LLM}};
  \draw[flow,shorten >=2.2pt] (prompt.south)--(llm.north);

  \path[attemptframe,fill=tealLight] (5.27,2.53) rectangle (7.76,4.57);
  \path[attemptframe,fill=tealLight] (8.55,2.53) rectangle (11.00,4.57);
  \node[font=\sffamily\small\bfseries,text=teal] at (6.51,4.82)
    {Attempt 1};
  \node[font=\sffamily\small\bfseries,text=pink] at (9.78,4.82)
    {Attempt 2};

  \node[font=\sffamily\scriptsize\bfseries,text=ink,anchor=east] (rollout1)
    at (5.12,4.20) {Rollout 1:};
  \node[font=\sffamily\scriptsize\bfseries,text=ink,anchor=east] (rollout2)
    at (5.12,3.55) {Rollout 2:};
  \node[font=\sffamily\scriptsize\bfseries,text=ink,anchor=east] (rollout3)
    at (5.12,2.90) {Rollout 3:};

  \node[attempt] (attempt11) at (6.30,4.20) {$y_{1,1}$};
  \node[result=red] (reward11) at (7.34,4.20) {$\times$};
  \node[attempt] (attempt12) at (9.55,4.20) {$y_{1,2}$};
  \node[result=green] (reward12) at (10.59,4.20) {$\checkmark$};

  \node[attempt] (attempt21) at (6.30,3.55) {$y_{2,1}$};
  \node[result=green] (reward21) at (7.34,3.55) {$\checkmark$};

  \node[attempt] (attempt31) at (6.30,2.90) {$y_{3,1}$};
  \node[result=red] (reward31) at (7.34,2.90) {$\times$};
  \node[attempt] (attempt32) at (9.55,2.90) {$y_{3,2}$};
  \node[result=red] (reward32) at (10.59,2.90) {$\times$};

  \coordinate (branch) at (3.18,3.55);
  \draw[flow] (llm.east)--(branch)--(3.18,4.20)--(rollout1.west);
  \draw[flow] (llm.east)--(rollout2.west);
  \draw[flow] (llm.east)--(branch)--(3.18,2.90)--(rollout3.west);
  \draw[flow] (rollout1.east)--(attempt11.west);
  \draw[flow] (rollout2.east)--(attempt21.west);
  \draw[flow] (rollout3.east)--(attempt31.west);

    \node[module=pink,fill=pinkLight,text width=6.4cm,
          font=\sffamily\scriptsize,inner xsep=7pt,inner ysep=3pt]
      (retryPrompt) at (8.15,6.05)
      {\RetryPromptText};
  \coordinate (retryJoinOne) at (8.15,4.20);
  \coordinate (retryJoinThree) at (8.15,2.90);
  \draw[retryline,rounded corners=8pt]
    (retryPrompt.south)--(8.15,4.20)--(8.34,4.20);
  \draw[retryline,rounded corners=8pt]
    (retryPrompt.south)--(8.15,2.90)--(8.65,2.90);
  \draw[retry] (reward11.east)--(attempt12.west);
  \draw[retry] (reward31.east)--(attempt32.west);

  \coordinate (collectorTop) at (11.30,4.20);
  \coordinate (collectorTurn) at (11.30,3.55);
  \coordinate (collectorMid) at (11.55,3.55);
  \coordinate (collectorBottom) at (11.30,2.90);
  \draw[collect,rounded corners=5pt]
    (reward12.east)--(collectorTop)--(collectorTurn)--(collectorMid);
  \draw[collect] (reward21.east)--(collectorMid);
  \draw[collect,rounded corners=5pt]
    (reward32.east)--(collectorBottom)--(collectorTurn)--(collectorMid);

  \node[normalizebox] (normalize) at (14.05,3.55)
    {\small\textbf{Normalize attempt rewards}\\[3pt]
     $\displaystyle \hat A_{i,j}=
       \frac{r_{i,j}-\mu_{\mathcal A}}{\sigma_{\mathcal A}+\varepsilon}$};
  \draw[flow] (collectorMid)--(normalize.west);

  \node[advantagebox] (advantage) at (14.05,1.48)
    {\small\textbf{Assign $\hat A_{i,j}$ to response $y_{i,j}$}\\[3pt]
     $\displaystyle \forall\,t\in y_{i,j}:\quad A_t=\hat A_{i,j}$};
  \draw[flow] (normalize)--(advantage);

  \node[grpobox] (grpo) at (7.65,1.48)
    {\small\textbf{GRPO objective}\\[3pt]
     $\displaystyle \max_{\theta}\;\mathcal L_{\mathrm{GRPO}}
       (\theta;\hat A_{i,j})$};
  \draw[flow] (advantage.west)--(grpo.east);

  \draw[update] (grpo.west)--(2.20,1.48)
    --node[right,font=\sffamily\scriptsize\bfseries,text=muted] {Update policy}
    (2.20,2.96)--(llm.south);
\end{tikzpicture}%
    }
    \caption{
    \textbf{Overview of ReAPO.} Each rollout begins with an independent response; failed responses receive a verifier-conditioned retry, while successful rollouts terminate. Advantages are computed for each attempt separately, ensuring that credit is assigned appropriately to both initial and retry attempts, before applying a policy update with the GRPO loss.
    }
    \label{fig:pipeline}
\end{figure}

Critically, GRPO relies on independently sampled rollouts to explore different reasoning trajectories, which can result in repeated exploration of the same incorrect paths and inefficient use of computational resources.  Indeed, if a model is unlikely to solve a given problem, it may well be the case that independent rollouts make the same mistake again and again, leading to no useful training signal for the model to learn from its mistakes and wasting valuable computational resources in generating redundant and uninformative tokens.
Recent work has demonstrated that introducing \emph{dependence} into the generation process can help improve coverage and potentially improve RLVR \citep{snell2024scalingllmtesttimecompute,madaan2023selfrefineiterativerefinementselffeedback,sprague2026skillfactoryselfdistillationlearningcognitive,kumar2024traininglanguagemodelsselfcorrect,li2026quasimotto}, but the best way to do so remains unclear.  We thus ask: \emph{How can we introduce structured dependence in the sampling process to improve exploration, and how should we assign credit to each attempt?}

In this work, we propose \textbf{Retry-Aware Policy Optimization (ReAPO)} (\Cref{alg:our_algorithm}), a simple yet effective modification to GRPO that explicitly encourages exploration by allowing the model to generate retries after failed attempts and assigning credit to each attempt in a theoretically principled manner.  ReAPO is compatible with any verifier-based reward function and does not require additional supervision or auxiliary objectives; moreover, ReAPO can be easily integrated into existing training pipelines, requiring minimal changes to existing code.

In more detail, our motivating question breaks into two parts whose combination yields ReAPO: (i) introducing beneficial dependence into generation and (ii) correctly assigning credit to individual attempts.  Prior work has demonstrated the efficacy of keeping information \emph{in context} to encourage language models to reflect on prior mistakes and correct themselves \citep{kumar2024traininglanguagemodelsselfcorrect,he2026selfdistillationzeroselfrevisionturns,bensal2025reflectretryrewardselfimproving,Guo_2025}.  We build on this idea and propose \emph{failure-conditioned repeated attempts}: given a prompt and response, we grade the answer and, if the answer is incorrect, append an instruction to the model to try again.  This process can then be repeated either until a correct answer is generated or until some maximum budget of attempts (or tokens) is reached.  Critically, unlike GRPO, where each attempt is fully independent, future attempts can learn from prior incorrect solutions, thereby achieving better coverage than mere independent sampling \citep{snell2024scalingllmtesttimecompute}.  While taking this idea to the extreme would entirely replace the parallel sampling of GRPO with serial sampling, in practice, the system's benefits of parallelization lead us to instead rollout multiple sequences of attempts independently, thereby introducing dependence within but not among distinct rollouts.
Failure-conditioning has additional benefits beyond exploration: serial sampling
can be viewed as a way to dynamically assign additional compute (measured in generated tokens) to more difficult problems where early attempts are less likely to achieve success. Differentially allocating compute based on problem difficulty has been observed to lead to significant improvements in token and wall-clock efficiency \citep{yu2025dapo,xiong2025reinforceadaadaptivesamplingframework,microsoft-ai-team-mai-thinking-1-building-a-hill-climbing-machine-2024}, and is likely partially responsible for ReAPO's efficiency.  

Given a group of sequences of attempts as described above, it is not obvious how to assign credit to a given sub-attempt within a larger sequence.  We propose to minimize the \emph{expected number of incorrect attempts before a correct solution is provided}, which is an instance of a \emph{stochastic shortest path} problem from the classical RL literature \citep{6dfe7762-0280-326e-88a5-877d93306e7b}.  Under natural simplifying assumptions, the resulting policy gradient computes \emph{response-level advantages}, i.e., we treat each attempt within a sequence independently and assign each attempt's advantage only to the tokens relevant to that attempt, thereby allowing different attempts within a single sequence to receive different advantages.  This approach encourages the model to produce correct solutions quickly while at the same time allowing the RL process to retain the benefits of conditioned sampling.  The resulting algorithm is represented visually in \Cref{fig:pipeline}, demonstrating how failed responses are incorporated into context and how correct and incorrect responses are scored.

We now summarize the remainder of the paper.  In \Cref{sec:preliminaries} we introduce the prerequisite notions to our derivation of ReAPO.  Then, in \Cref{sec:method}, we formally derive ReAPO's objective under natural simplifying assumptions on the generation process.  In \Cref{sec:experiments}, we investigate ReAPO empirically across four models from two families (0.6B to 4B parameters) and eight mathematical and scientific reasoning benchmarks, comparing against GRPO, compute-matched GRPO variants and other recent baselines, and ablating each component. We briefly discuss related work in \Cref{sec:conclusion}, and defer further discussion, additional experiments and all proofs to the appendix.

\section{Preliminaries}
\label{sec:preliminaries}

\begingroup
\setlength{\abovedisplayskip}{4pt}
\setlength{\belowdisplayskip}{4pt}
\setlength{\abovedisplayshortskip}{4pt}\setlength{\belowdisplayshortskip}{4pt}

We now provide a brief overview of RLVR as well as the popular Group Relative Policy Optimization (GRPO) algorithm.
As is common in RL for language models, we consider a language model to be a \emph{policy} $\pi_\theta$ that maps a prompt $x$ to a distribution over responses $y$.  In this work, we always assume that the model is autoregressive, i.e., the response is generated one token at a time so that for each time step $t$, the model produces a token $y_t \sim \pi_\theta(\cdot | x, y_{<t})$ conditioned on the prompt and all previously generated tokens.  To train the model, we assume that there is a \emph{verifier} $r$ that can determine whether a given response $y$ is correct for a given prompt $x$, and our goal is to optimize the model to maximize the probability of generating correct responses.
More precisely, we let $\mathcal{D}=\{x\}$ be a dataset of problems $x$. Given a problem $x$, the model generates a response $y \sim \pi_\theta(\cdot | x)$, and the verifier $r$ produces a binary reward $r(x,y) \in \{0,1\}$ indicating whether the response is correct, with the goal of maximizing $\mathbb E_{X \sim \mathcal D} \mathbb E_{y \sim \pi_{\theta(\cdot | x)}}\left[ r(x,y) \right]$.

\paragraph{GRPO.} The most popular approach to RLVR is \emph{Group Relative Policy Optimization} (GRPO)~\cite{shao2024deepseekmathpushinglimitsmathematical,Guo_2025}. For a given prompt $x$, GRPO samples a group $\{y_i\}_{i=1}^G\sim\pi_{\theta_{\mathrm{old}}}(\cdot\mid x)$ of $G$ responses from the current policy and obtains rewards $r_i=r(x,y_i)$ for each response. The rewards are then centered and normalized within the group to produce the response-level advantages
\begin{align}\label{eq:grpo_advantage}
\hat A_i=\frac{r_i-\mu_r}{\sigma_r+\epsilon},\qquad
\mu_r=\frac{1}{G}\sum_{i=1}^G r_i,\qquad
\sigma_r=\sqrt{\frac{1}{G}\sum_{i=1}^G(r_i-\mu_r)^2},
\end{align}
where $\epsilon>0$ provides numerical stability.
GRPO maximizes the PPO objective~\cite{schulman2017proximalpolicyoptimization}
\begin{align}
\mathcal{J}_{\mathrm{GRPO}}(\theta)
=\mathbb{E}\left[
\frac{1}{T}\sum_{i=1}^G\sum_{t=1}^{|y_i|}
\left(
\min\left\{
\rho_{i,t}(\theta)\hat A_i,\,
\operatorname{clip}\left(\rho_{i,t}(\theta),1-\varepsilon,1+\varepsilon\right)\hat A_i
\right\}
-\beta d_{\mathrm{KL}}^{i,t}(\theta)
\right)
\right],
\label{eq:grpo_loss}
\end{align}
where $\rho_{i,t}(\theta)=\nicefrac{\pi_\theta(y_{i,t}\mid x,y_{i,<t})}{\pi_{\theta_{\mathrm{old}}}(y_{i,t}\mid x,y_{i,<t})}$ for token $y_{i,t}$ is the importance weight, $T=\sum_{i=1}^G|y_i|$ is the total number of response tokens, $\operatorname{clip}(z,a,b)=\min\{\max\{z,a\},b\}$ clips $z$ to the interval $[a,b]$,  
$\varepsilon>0$ is the clipping parameter, $\beta\geq0$ is the KL-regularization coefficient, and $d_{\mathrm{KL}}^{i,t}(\theta)$ is the per-token KL divergence from the reference policy.  Following recent work \citep{yu2025dapo,tajwar2026maximumlikelihoodreinforcementlearning,hu2025openreasonerzeroopensourceapproach}, we set $\beta=0$ and therefore do not apply explicit reference-policy KL regularization in our experiments; moreover we always have $\theta = \theta_{\mathrm{old}}$.

\paragraph{Stochastic Shortest Paths.} SSP problems are a classical subfield of RL concerned with undiscounted MDPs whose objective is to reach a terminal state while minimizing cumulative cost \citep{6dfe7762-0280-326e-88a5-877d93306e7b}.  In the context of LLM reasoning, we conceptualize each `attempt' as an action and the concatenation of the problem and previous attempts as a state, with the cumulative cost being the total number of attempts until a correct solution (the terminal state) is reached; for further details on classical SSPs, see \citet{bertsekas1996neuro}.

\endgroup

\begin{algorithm}[t]
\DontPrintSemicolon
\SetKwInOut{Require}{Require}

\caption{Retry-Aware Policy Optimization (ReAPO)}
\label{alg:our_algorithm}

\Require{Base model $\pi_\theta$, dataset batch $\mathcal{D} = \{(x, a)\}_{i=1}^N$, verifier $r(y, a)$, group size $G$, max attempts $K$, retry prompt $P_{\text{retry}}$, and stabilizing value $\epsilon > 0$.}

\For{$(x, a) \in \mathcal{D}$}{
    \scalebox{0.93}[1]{\textcolor{commentblue}{\%\% Multi-Attempt Group Rollout}}\;
    \For{$g = 1 \dots G$}{
        Initialize $C \leftarrow [x]$\;
        \For{$k = 1 \dots K$}{
            Sample $y_{g, k} \sim \pi_\theta(\cdot \mid C)$\;
            Compute $r_{g,k} \leftarrow r(y_{g, k}, a)$\;
            \If{$r_{g,k} = 1$ \text{or} $k = K$}{
                $k_g \leftarrow k$\;
                \textbf{break}\;
            }
            Update context $C \leftarrow C \cup y_{g, k} \cup P_{\text{retry}}$\;
        }
    }
    
    \vspace{0.2cm}
    \scalebox{0.93}[1]{\textcolor{commentblue}{\%\% Multi-Attempt Advantage Estimation \& Optimization}}\;
    
    Construct $\mathcal{A} = \{ r_{g,k} \mid g \in 1\dots G, \, k \in 1\dots k_{g} \}$, where:\;
    Compute mean $\mu_\mathcal{A}$ and standard deviation $\sigma_\mathcal{A}$\;
    
    \For{$r_{g,k} \in \mathcal{A}$}{
        $\hat{A}_{g,k} \leftarrow \frac{r_{g,k} - \mu_\mathcal{A}}{\sigma_\mathcal{A} + \epsilon}$ \hfill \textcolor{commentblue}{(\%\% Compute relative advantage)}\;
        Assign $\hat{A}_{g,k}$ uniformly to all tokens in the response $y_{g, k}$\;
    }
}
Update $\theta$ by taking a gradient step in the direction of \eqref{eq:reapo-objective}. \;
\end{algorithm}

\section{ReAPO: Retry-Aware Policy Optimization}
\label{sec:method}

In this section we introduce our algorithm \textbf{Retry-Aware Policy Optimization (ReAPO)}, summarized in \Cref{alg:our_algorithm}. 
ReAPO has two key components on top of the standard GRPO pipeline: (1) \emph{verifier-conditioned exploration}, which allows the model to generate retries after failed attempts, and (2) \emph{attempt-level credit assignment}, which correctly assigns credit to each attempt based on its relative performance within the group.  We describe each component in detail below.

\subsection{Verifier-Conditioned Exploration}
In GRPO, each model rollout is sampled independently and thus can repeatedly explore the same (incorrect and useless) reasoning trajectories, wasting compute and failing to learn from its mistakes.  Motivated by prior work \citep{sprague2026skillfactoryselfdistillationlearningcognitive,kumar2024traininglanguagemodelsselfcorrect,chung2025learningfailuresmultiattemptreinforcement}, we instead propose \emph{verifier-conditioned exploration}.
More precisely, for each problem $x$, ReAPO begins with the same initial group of $G$ rollouts as standard GRPO. Each rollout first generates an independent response $y_{g,1}$. If the verifier marks this response as correct, the rollout terminates. If it is \emph{incorrect}, however, the response is appended to the context together with a \emph{retry prompt}, that explicitly instructs the model to reconsider its previous answer and attempt a different reasoning strategy.  Our ablations demonstrate that the exact wording of the retry prompt is not critical, and we use ``\emph{Your answer is wrong. Reflect on your previous answer and try something else.}'' for most of our experiments.
The model then generates a second response $y_{g,2}$ conditioned on the problem, its previous response, and the verifier feedback. This process can continue until the model produces a correct response or reaches the maximum number of (per-problem) attempts $K$, or hits its token budget.

Because retries are generated only after failures, ReAPO dynamically expands the training group according to the observed difficulty of the problem under the current policy, which is known to help improve efficiency in modern RL for language modeling \citep{zhang2025improvingsamplingefficiencyrlvr,xiong2025reinforceadaadaptivesamplingframework,yang2026depthbreadthsynergyrlvrunlocking}. Problems for which most initial responses are correct receive few additional generations, whereas problems with many incorrect responses receive more retries.  This conditioned sampling allows for greater exploration of the model's reasoning space; for example in \Cref{fig:aime25-sequential-vs-parallel} in the appendix, we observe a substantial reduction in the required number of attempts in order to achieve a target pass@$k$ accuracy relative to na{\"i}ve independent sampling, even for just the base model.

\subsection{Assigning Credit with Multiple attempts}

Verifier-conditioned exploration helps solve the problem of repeated incorrect reasoning by allowing the model to condition its retries on previous failures, but it raises the question of how to properly assign credit to individual attempts within a multi-attempt trajectory.  We propose minimizing the expected number of failed attempts to reach a correct solution as a natural objective.  Note that in so doing, we are encouraging the model to produce correct solutions earlier, including on the first attempt, which is the desideratum during evaluation.

More formally, fix a problem $x$ and let $y_k$ denote the $k$-th attempt generated for this problem, with $r_k = r(x, y_k)$ the binary reward; letting $\tau = \inf\left\{ k \geq 1 \mid r_k = 1 \right\}$ denote the index of the first successful attempt, we aim to minimize $\ee_{\pi_\theta}\left[ \tau - 1 \right]$.  This is an example of a \emph{Stochastic Shortest Path} (SSP) problem, where the terminal state is a correct answer and the cost is constant per action \cite{6dfe7762-0280-326e-88a5-877d93306e7b}.  In order to derive a practical policy gradient, we make the following simplifying assumption: the rewards of attempts are sufficient statistics for the probability of a future attempt being correct, i.e., incorrect attempts affect the probability of a future attempt being correct only through the fact that they are incorrect.  This is certainly an overly conservative assumption as, in practice, knowing that a prior failed attempt is wrong is likely to help the model attempt something qualitatively different, but its simplicity allows for a simple derivation.  To state the next result, we introduce the following notation.  Let $q_m(\theta) = \pp_{\pi_\theta}(r_m = 1 | r_j = 0 \text{ for } j < m)$ be the probability of success on the $m^{th}$ attempt and let $\alpha_m(\pi_\theta) = \pp\left( \tau \geq m \right)$ be the probability that we need at least $m$ attempts in order to be successful.  Finally, let
\begin{align}\label{eq:cmtheta}
    c_m(\pi_\theta) = \sum_{k = m}^\infty \prod_{j = m+1}^k \left( 1 - q_j(\theta) \right)
\end{align}
denote the expected number of additional attempts until a correct response is obtained, starting from attempt $m+1$ given that the first $m$ attempts have failed.  We have the following result.

\begin{proposition}\label{prop:policy_gradient_decomposition}
    Fix a prompt $x$ and let $\pi_\theta$ be the policy generating attempts for this problem. If $\tau$ denotes the index of the first successful attempt and $c_m(\pi_\theta)$ is defined as in \eqref{eq:cmtheta}, then the policy gradient of the expected number of failed attempts can be decomposed as
    \begin{align}\label{eq:full_policy_gradient}
        - \nabla_\theta \ee_{\pi_\theta}[\tau - 1] = \sum_{m = 1}^\infty \alpha_m(\pi_\theta) c_m(\pi_\theta) \nabla q_m(\theta).
    \end{align}
\end{proposition}
The proof of this (and all other) results can be found in \Cref{app:theory}.  Unfortunately, the decomposition in \eqref{eq:full_policy_gradient} possesses two deficiencies.  First, it requires knowledge of $c_m(\pi_\theta)$, which depends on the probabilities of success for all future attempts and is therefore not directly observable or easily estimated.  Second, the infinite sum may be impractical to compute in practice, necessitating truncation or approximation strategies.  Thus, in order to derive a practical policy gradient estimator, we make the further assumption that $c_m(\pi_\theta)$ is \emph{independent of } $m$.  Under this assumption, a natural advantage estimation strategy recovers the form of \eqref{eq:full_policy_gradient} with a constant $c_m(\pi_\theta)$.

Indeed, suppose that we roll out $G$ sequences of attempts per problem and denote by \iftoggle{arxiv}{
    \begin{align}
        \cA(x) = \left\{ r_{g,k} | g \in [G] \text{ and } k \geq 1 \right\}
    \end{align}
}{$\cA(x) = \left\{ r_{g,k} | g \in [G] \text{ and } k \geq 1 \right\}$} the collection of rewards from all generated attempts.  We claim that a GRPO-like advantage estimation is proportional in expectation to the true policy gradient of \eqref{eq:full_policy_gradient} under the assumption that $c_m(\pi_\theta)$ is independent of $m$.
\begin{proposition}\label{prop:time_invariant_policy_gradient}
    Let $\cA(x)$ be as above. Assume that $c_m(\pi_\theta) = C(\pi_\theta)$ for all $m$, and let
    \begin{align}
        \nabla_\theta \mathcal{L}(\pi_\theta) = \ee_{\pi_\theta}\left[ \frac 1G \sum_{g=1}^G \sum_{k=1}^{\tau_g} r_{g,k} \cdot \nabla_\theta \log \pi_\theta(y_{g,k} \mid h_{g,k-1}) \right], %
    \end{align}
    where $h_{g,k-1} = (x,y_{g,1},\dots,y_{g,k-1})$. Then it holds that
    $
        -\nabla_\theta \ee_{\pi_\theta}[\tau - 1] = C(\pi_\theta) \cdot \nabla_\theta \mathcal{L}(\pi_\theta).
$
\end{proposition}
Under the simplifying assumptions above, weighting each generated attempt by its own unnormalized reward yields an update proportional to the negative gradient of the expected number of failed attempts. This result motivates the attempt-level credit assignment used in ReAPO.
Critically, we are only applying the attempt-level advantage to the tokens making up that attempt, rather than distributing the advantage across all tokens in the trajectory \citep{kumar2024traininglanguagemodelsselfcorrect,
chung2025learningfailuresmultiattemptreinforcement}.

\begin{figure}[t]
    \centering
    \includegraphics[width=\textwidth]{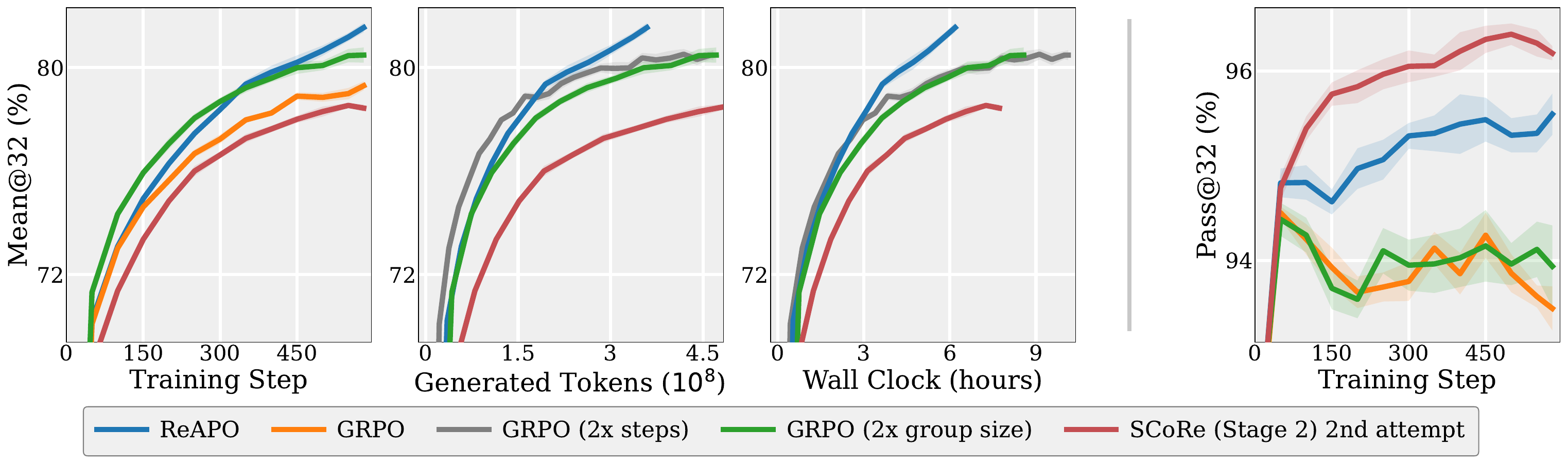}
    \caption{\textbf{Training efficiency on GSM8K-Platinum with Qwen3-0.6B-Base.} The first three panels plot accuracy (mean@32) against optimization steps, cumulative generated response tokens, and wall-clock time, respectively. The final panel plots Pass@32 against optimization steps. The \textsc{GRPO-2$\times$Steps} control extends standard GRPO to 1,160 steps in the compute-normalized panels. Curves show means across eight seeds and shaded regions show standard errors. SCoRe-Stage2 is evaluated on its second response, which is generated regardless of the first response's correctness and without verifier feedback at inference. Every other method is evaluated on its first response.}
    \label{fig:qwen06-efficiency}
\end{figure}

\begin{figure}[t]
    \centering
    \includegraphics[width=\textwidth]{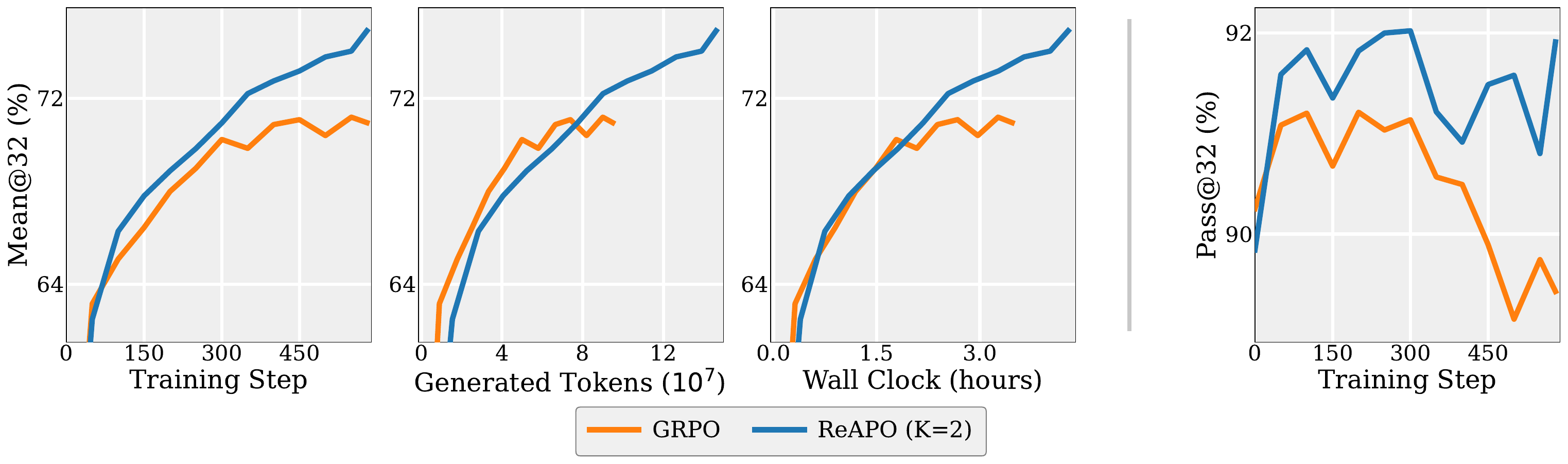}
    \caption{\textbf{Training efficiency on GSM8K-Platinum with OLMo-2-0425-1B-SFT.} The first three panels plot first-attempt Mean@32 against optimization steps, cumulative generated response tokens, and wall-clock time; the final panel plots Pass@32 against optimization steps. Results are from a single training seed.}
    \label{fig:olmo1b}
\end{figure}

\subsection{Implementing ReAPO}

We now turn to the practical implementation of ReAPO, making
two key design modifications to the method suggested by \Cref{prop:time_invariant_policy_gradient}.  First, practical constraints limit the maximum number of attempts per problem that we can actually sample, due to the quadratic increase in computational cost and linear increase in memory with the number of attempts, along with the reduced parallelization.  We thus enforce a \emph{hard cap} $K$ on the number of attempts per problem, which acts to bias the policy gradient estimate by truncating the infinite sum in \eqref{eq:full_policy_gradient} to the first $K$ attempts.  A theoretical account of this bias can be found in \Cref{app:theory}.  Second, we follow GRPO and center and normalize the advantages by the standard deviation (offset by a small term to ensure numerical stability) \citep{shao2024deepseekmathpushinglimitsmathematical}.  Exploring the effect of ablating this normalization in the style of Dr.GRPO \cite{liu2025understandingr1zeroliketrainingcritical} is an interesting direction for future research, although we briefly discuss the theoretical consequences of normalization in \Cref{app:theory}; centering, meanwhile, introduces a small bias of order $\nabla_\theta \ee_{\pi_\theta}[\nicefrac{1}{\tau}]$ (see \Cref{app:theory}).

After these modifications we are left with the following algorithm.  Let $r_{g,k}$ denote the reward obtained for the $k$-th attempt of the $g$-th group.  Letting $\hat{A}_{g,k}$ be defined as in \eqref{eq:grpo_advantage}, where the mean  and standard deviation are computed over all $r_{g,k}$, and $T=\sum_{g=1}^G\sum_{k=1}^{\tau_g}|y_{g,k}|$ is the total number of generated response tokens across all attempts, the ReAPO update direction is given by
\begin{align}\label{eq:reapo-objective}
    \nabla_\theta \mathcal{J}_{\mathrm{ReAPO}}(\theta) = \ee_{x\sim\mathcal D}\ee_{\pi_{\theta}}\left[ \frac{1}{T} \sum_{g = 1}^G \sum_{k = 1}^{\tau_g} \hat{A}_{g,k} \cdot \nabla_\theta \log \pi_\theta(y_{g,k}\mid x,y_{g,<k}) \right]
\end{align}
As clarified in \Cref{sec:preliminaries}, in our experiments we always roll out on-policy (i.e. $\pi_{\theta_{\mathrm{old}}} = \pi_\theta$) and omit the KL regularization term; both can be incorporated into \eqref{eq:reapo-objective} in a straightforward manner.

To reiterate, the advantage $\hat A_{g,k}$ is \emph{assigned only to tokens belonging to the response} $y_{g,k}$. Thus, when an incorrect first attempt is followed by a correct retry, the failed response retains its lower advantage while the successful retry receives a higher advantage. The retry does not retroactively reward the reasoning that led to the initial failure. Likewise, a retry that repeats an incorrect answer receives another zero reward and a negative relative advantage whenever the pooled reward collection contains any successful attempts. If every generated attempt receives the same reward, all normalized advantages are zero and the problem contributes no policy-gradient signal, much as in GRPO.  
The pseudocode for ReAPO is presented in Algorithm~\ref{alg:our_algorithm}.

\section{Experiments}
\label{sec:experiments}

\begin{table}[t]
\centering
\small
\caption{\textbf{First-attempt accuracy across model scales and benchmarks.} Qwen3-1.7B-Base results are mean $\pm$ standard error across seven seeds at the end of training; Qwen3-4B-Base results use three seeds and are evaluated at step 500. 
}
\setlength{\tabcolsep}{2.5pt}
\renewcommand{\arraystretch}{1.08}
\scalebox{0.92}{%
\begin{tabular}{@{}l*{7}{c}@{\hspace{8pt}}c@{}}
\toprule
\textbf{Model / Method} & \textbf{\shortstack{Beyond\\AIME}} & \textbf{\shortstack{AIME\\2024}} & \textbf{\shortstack{AIME\\2025}} & \textbf{\shortstack{HMMT\\2025}} & \textbf{\shortstack{MATH\\500}} & \textbf{Minerva} & \textbf{\shortstack{GPQA\\Diamond}} & \textbf{\shortstack{Avg.\\~}} \\
\midrule
Qwen3-1.7B & $1.2$ & $4.2$ & $3.1$ & $0.2$ & $57.2$ & $21.9$ & $10.8$ & $14.1$ \\
GRPO & $1.3 \pm 0.1$ & $11.5 \pm 0.4$ & $6.2 \pm 0.3$ & $1.3 \pm 0.1$ & $72.0 \pm 0.5$ & $30.9 \pm 0.2$ & $24.1 \pm 0.7$ & $21.1 \pm 0.3$ \\
\textbf{ReAPO} (ours) & $\mathbf{1.6 \pm 0.1}$ & $\mathbf{13.7 \pm 0.3}$ & $\mathbf{7.4 \pm 0.3}$ & $1.3 \pm 0.2$ & $\mathbf{72.9 \pm 0.3}$ & $\mathbf{31.7 \pm 0.2}$ & $\mathbf{25.6 \pm 0.3}$ & $\mathbf{22.0 \pm 0.1}$ \\
\midrule
Qwen3-4B & $4.1$ & $9.4$ & $6.4$ & $0.9$ & $67.9$ & $27.4$ & $17.2$ & $19.0$ \\
GRPO & $7.6 \pm 0.3$ & $19.5 \pm 1.1$ & $16.5 \pm 1.7$ & $5.5 \pm 0.8$ & $83.0 \pm 0.6$ & $40.8 \pm 0.5$ & $28.6 \pm 0.6$ & $28.8 \pm 0.7$ \\
\textbf{ReAPO} (ours) & $\mathbf{8.8 \pm 0.4}$ & $\mathbf{22.4 \pm 0.7}$ & $\mathbf{19.0 \pm 0.9}$ & $\mathbf{9.6 \pm 0.5}$ & $\mathbf{85.3 \pm 0.3}$ & $41.3 \pm 0.3$ & $\mathbf{30.4 \pm 0.7}$ & $\mathbf{31.0 \pm 0.1}$ \\
\bottomrule
\end{tabular}%
}
\label{tab:large-model-results}
\end{table}

We now present a comprehensive evaluation of ReAPO, comparing it against standard GRPO and other relevant baselines across model scales, mathematical reasoning benchmarks, and compute budgets.  We begin by describing our empirical setup (deferring additional details to  \Cref{app:experimental-details}), then we examine the training efficiency in a number of configurations on a small model, before scaling up to larger models and more challenging benchmarks.  We then thoroughly investigate the impact of the choices that underlie ReAPO's design, including the advantages of serial sampling.

\subsection{Experimental Setup}

\paragraph{Models and training data.}
We train OLMo-2-0425-1B-SFT, Qwen3-0.6B-Base, Qwen3-1.7B-Base, and Qwen3-4B-Base~\cite{yang2025qwen3technicalreport}. All methods default to $G=16$ rollouts per problem, which is the default group size used by recent works \cite{tajwar2026maximumlikelihoodreinforcementlearning}, and $K = 2$ attempts in ReAPO due to the limitations of context length. The smallest model is trained on GSM-8K~\citep{cobbe2021trainingverifierssolvemath} with a cumulative limit of 2,048 generated response tokens per trajectory and the 1.7B/4B models on Polaris-53K~\cite{Polaris2025} with a cumulative limit of 4,096 generated response tokens per trajectory. Within each model scale, all methods use the same per-trajectory response-token limit, shared across attempts. All hyperparameters are provided in Tables~\ref{tab:qwen06-hyperparameters} and \ref{tab:qwen-large-hyperparameters} in the appendix.

\paragraph{Evaluation.}
We validate on standard mathematical and scientific reasoning benchmarks. For Qwen3-0.6B-Base, we evaluate on GSM8K-Platinum~\cite{vendrow2025largelanguagemodelbenchmarks}, while for Qwen3-1.7B-Base and Qwen3-4B-Base, we evaluate on BeyondAIME~\cite{bytedance_seed_2025_beyondaime}, AIME 2024, AIME 2025~\cite{aime24,aime25}, HMMT 2025~\cite{dekoninck2026matharena}, MATH-500~\cite{lightman2023letsverifystepstep,hendrycks2021measuringmathematicalproblemsolving}, Minerva~\cite{lewkowycz2022solvingquantitativereasoningproblems}, and GPQA-Diamond~\cite{rein2024gpqa}.

\paragraph{Metrics.}
Unless stated otherwise, \textbf{we evaluate only the first response, without verifier feedback at inference}. 
\emph{Mean@32} is the average accuracy of 32 evaluated responses per problem, while \emph{Pass@32} is the fraction of problems solved by at least one of those responses. These metrics use first responses for all methods except SCoRe-Stage2, for which they use second responses, which are generated regardless of whether its first response is correct. Error bars denote the standard error across seeds.

\subsection{ReAPO Efficiently Outperforms GRPO}
\label{sec:small_model_results}

We begin by examining the training efficiency of ReAPO relative to GRPO on the small Qwen3-0.6B-Base model.  Critically, while the \emph{maximum number of tokens} per trajectory is fixed across methods at each model scale, due to the reattempts present in ReAPO, it may well be that the \emph{average number of tokens} per trajectory is much higher.  Thus, in order to form a fair comparison among methods, we compare the training efficiency in terms of optimizer steps, the total number of generated tokens, and the wall clock time of each run.  In \Cref{fig:qwen06-efficiency}, ReAPO outperforms GRPO at every training step; measured by generated tokens and wall-clock time, GRPO is competitive early in training, but ReAPO overtakes it and reaches a higher final accuracy.  Prior work has found that RL can often improve pass@1 performance without a concomitant increase in higher pass@k accuracy \cite{yue2025doesreinforcementlearningreally}; we thus also examine the effect of ReAPO on pass@32 throughout training and find directionally similar gains.  Because pass@32 is almost saturated at this benchmark, however, the magnitude of the gains is much smaller.

In addition to GRPO, we consider three baselines that isolate the source of ReAPO's gains. First, we train GRPO for twice as many steps, substantially increasing both generated-token and wall-clock compute. Even with this additional compute, GRPO plateaus below ReAPO (\Cref{fig:qwen06-efficiency}). Second, to test whether the gains arise simply from generating more responses rather than from serial sampling, we double the GRPO group size to $G=32$. This matches the maximum number of attempts that ReAPO can generate from $G=16$ initial rollouts with $K=2$, but keeps all responses independent and doubles the compute budget. This baseline again plateaus below ReAPO and is substantially less token-efficient; indeed, \textbf{ReAPO requires 39\% fewer generated tokens to reach the peak accuracy of GRPO with twice the group size and 31\% fewer to reach the best performance of GRPO with twice the steps}.

\begin{figure}[t]
    \centering
    \includegraphics[width=\textwidth]{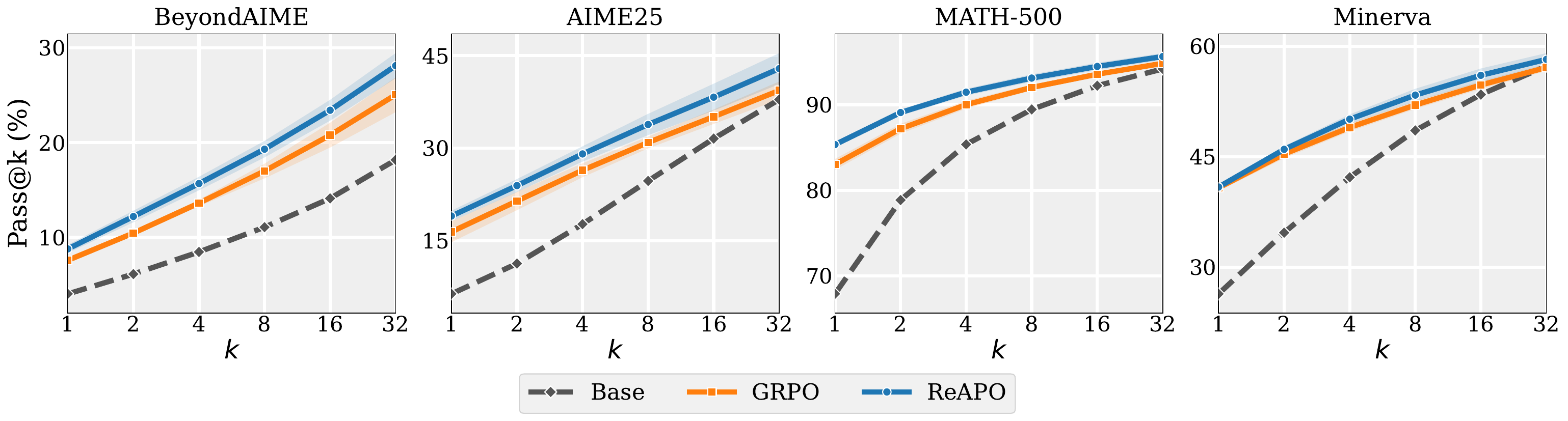}
    \caption{\textbf{First-attempt Pass@$k$ for Qwen3-4B-Base.} Pass@$k$ accuracy for each of BeyondAIME, AIME 2025, MATH-500, and Minerva for ReAPO, GRPO, and the base model. GRPO and ReAPO curves are means across three training seeds, with shaded regions denoting standard errors.
    }
    \label{fig:qwen_4b_best32}
\end{figure}

Third, because our comparison concerns how multi-turn training assigns rewards across attempts, we evaluate a SCoRe Stage-II-only baseline~\citep{kumar2024traininglanguagemodelsselfcorrect}\footnote{SCoRe's Stage I is a separate RL initialization that improves second-attempt performance while KL-constraining first attempts toward the base policy, which we omit; all algorithms start from the same checkpoint.}. Starting from Qwen3-0.6B-Base, this baseline jointly optimizes both attempts and a progress bonus to the second-attempt reward. Its second response is generated regardless of the first response's correctness and requires no verifier feedback at inference, so we report its second-attempt performance. Our results are reported in \Cref{fig:qwen06-efficiency}: due to the lack of verifier-conditioning, SCoRe is substantially less token- and compute-efficient relative to ReAPO and the different credit assignment does not encourage exploration sufficiently to match ReAPO's performance in mean@32, although the pass@32 performance is slightly improved.  A summary of these results can be found in Table~\ref{tab:qwen06-results} in the appendix. To test whether these gains extend beyond the Qwen family, we repeat the GSM8K experiment with OLMo-2-0425-1B-SFT (\Cref{fig:olmo1b}). ReAPO again improves both Mean@32 and Pass@32 over GRPO throughout training, and surpasses GRPO in token and wall-clock-normalized comparisons within GRPO's budget.

\subsection{Gains Persist at Larger Scales}
\label{sec:large_models_results}

While we showed above that ReAPO significantly improves training efficiency relative to GRPO on the small Qwen3-0.6B-Base model, it is important to verify whether these gains persist as we scale to larger models and more challenging benchmarks. In Table~\ref{tab:large-model-results}, we report (first-attempt) accuracy at the selected checkpoints, where at both the 1.7B and 4B scales we see that ReAPO improves on GRPO (and the base model) on nearly every benchmark. Moreover, in \Cref{fig:qwen_4b_best32}, we show that ReAPO maintains its advantage over GRPO even as the inference-time sampling budget increases for the 4B scale.  A similar trend is observed for the 1.7B scale, deferred to Appendix Figure~\ref{fig:qwen_1.7b_best32}.  More fine-grained results are reported in Table~\ref{tab:pass32-results} in the appendix.

\begin{figure}[t]
    \centering
    \includegraphics[width=\textwidth]{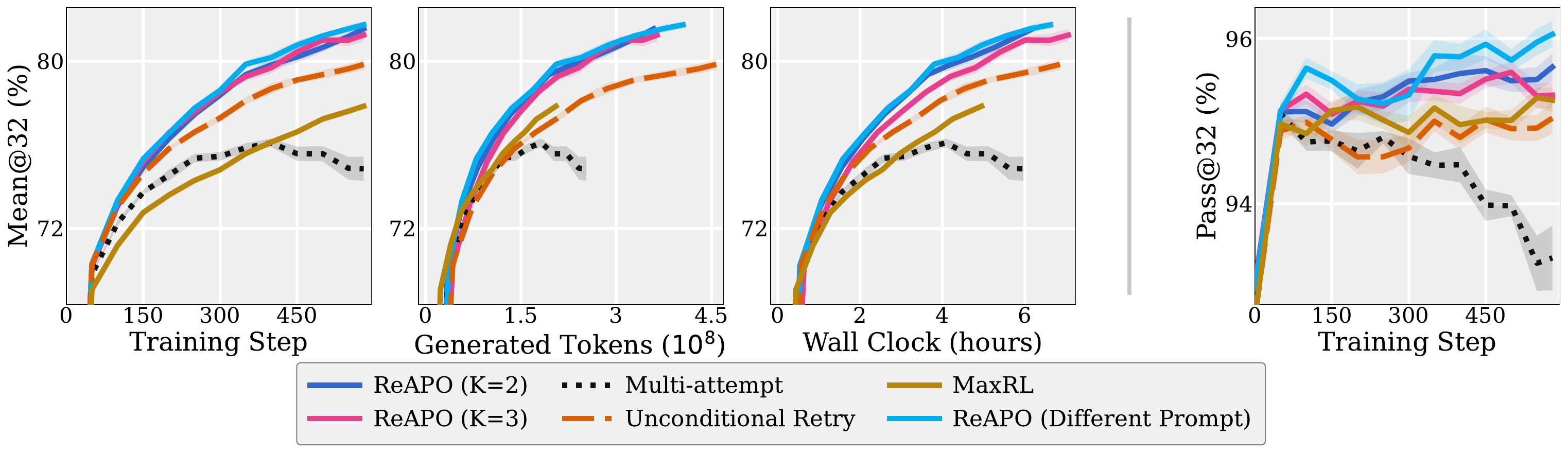}
    \caption{\textbf{Ablations on Qwen3-0.6B-Base and GSM8K-Platinum.} We compare ReAPO with $K \in \{2,3\}$, ReAPO with an alternative retry prompt, an unconditional-retry variant that generates a second attempt regardless of whether the first attempt is correct,
    a multi-turn variant that assigns one advantage to the complete trajectory, and MaxRL. The first three panels report first-attempt Mean@32 against training steps, cumulative generated response tokens, and wall-clock time, respectively; the final panel reports first-attempt Pass@32 against training steps.}
    \label{fig:qwen06-ablation}
\end{figure}

\subsection{Attempt-Level Credit and Conditional Retries Matter}

\paragraph{Additional baselines.}  In \Cref{fig:qwen06-ablation}, we compare ReAPO to two additional baselines. MaxRL \citep{tajwar2026maximumlikelihoodreinforcementlearning} is a recent method that aims to improve pass@k performance by modifying the advantage computation and has seen success in mathematical reasoning tasks.  Multi-Attempt GRPO \cite{chung2025learningfailuresmultiattemptreinforcement} treats an entire sequence of attempts as a single multi-turn rollout. The model repeatedly attempts the problem, conditioned on its previous attempts, until it either produces a correct answer or reaches the maximum number of attempts; the entire trajectory is assigned a single reward. Consequently, failed attempts that precede a successful retry receive the same credit as the successful attempt itself. This baseline isolates the benefit of ReAPO's attempt-level credit assignment from that of verifier-conditioned multi-turn sampling alone.  ReAPO consistently outperforms both baselines, demonstrating the utility of our proposed credit assignment.

\paragraph{Increasing the reattempt budget.} One of the chief hyperparameters of ReAPO is the attempt budget $K$, which determines the maximum number of attempts per trajectory. While increasing $K$ offers more opportunities to correct mistakes, it comes at the cost of more tokens per trajectory. Our results show that increasing to $K=3$ on the 0.6B model yields limited gains (shown in \Cref{fig:qwen06-ablation} and \Cref{tab:qwen06-results}). We believe this is because the 0.6B model saturates on GSM8K by the second attempt, and so the third provides little learning signal. More importantly, by the end of training, more than $90\%$ of trajectories succeed within the first two attempts, suggesting that the third attempt is rarely needed. On the other hand, the 1.7B model yields a more consistent gain in mathematical reasoning benchmarks (shown in \Cref{tab:qwen17b-seed42-eval}), although this gain does not appear to hold for scientific reasoning benchmarks (GPQA-Diamond). Since the result of the 1.7B is on a single seed, we view this as suggestive.

\paragraph{Sensitivity to prompt and verifier conditioning.}  Another hyperparameter of ReAPO is how we prompt the model after a failure.  In \Cref{fig:qwen06-ablation}, we compare a different prompt (given in \Cref{app:ablation-setup}), and find minimal difference in performance.
We then consider a ReAPO variant that generates a second attempt regardless of first-attempt correctness. 
Because this variant retries even after a correct first attempt, the retry is no longer conditioned on a known failure, and part of the additional generation is spent on problems that are already solved. It uses more tokens than ReAPO and reaches lower accuracy (\Cref{fig:qwen06-ablation}).

\subsection{Serial Sampling Encourages Exploration}
\label{sec:retry-pivot}

We conclude our experiments by examining how serial sampling affects exploration. We first examine the extent to which retries lead to more diverse solutions relative to independent sampling.  In \Cref{fig:retry-repetition} (left), we compare the fraction of duplicate (independent) first attempts for ReAPO and GRPO throughout training and see that these are roughly comparable.  In contradistinction to independent sampling, in the middle plot of that figure, we examine the fraction of second-attempt answers that exactly match the failed first-attempt answer, and observe that this quickly falls to near zero, demonstrating that the second attempt is highly likely to explore a different solution.  A table summary of commonly generated phrase can be found in \Cref{tab:qwen4b-attempt-phrases} and \Cref{tab:attempt-phrases}.

Finally, we investigate whether the increased success of ReAPO comes at the cost of solution coverage by examining how pass@$k$ accuracy improves with $k$ for different sampling budgets at the end of training.  We find in \Cref{tab:qwen06-results} that ReAPO maintains or improves solution coverage relative to GRPO and matches MaxRL, which is designed specifically to improve pass@$k$ performance.

\definecolor{reapoPurple}{HTML}{615CED}
\definecolor{reapoPurpleBg}{HTML}{F1F0FF}
\definecolor{reapoBlueBg}{HTML}{EDF7FC}
\definecolor{reapoRed}{HTML}{C94E5B}
\definecolor{reapoRedBg}{HTML}{FFF2F3}
\definecolor{reapoGreen}{HTML}{27865E}
\definecolor{reapoGreenBg}{HTML}{EFFAF5}
\definecolor{reapoGold}{HTML}{A67500}
\definecolor{reapoGoldBg}{HTML}{FFF7D6}
\definecolor{reapoRule}{HTML}{D9DDE3}

\tcbset{
  reapoouter/.style={
    enhanced,
    boxrule=0.7pt,
    arc=1.8mm,
    colback=white,
    colframe=reapoRule,
    left=3mm,
    right=3mm,
    top=1.2mm,
    bottom=1.2mm,
  },
  reapoattempt/.style={
    enhanced,
    boxrule=0.7pt,
    arc=2.2mm,
    left=2.5mm,
    right=2.5mm,
    top=0.9mm,
    bottom=0.9mm,
    before skip=0pt,
    after skip=0pt,
  }
}

\newcommand{\reapoQwenIcon}{%
  \includegraphics[width=7.5mm]{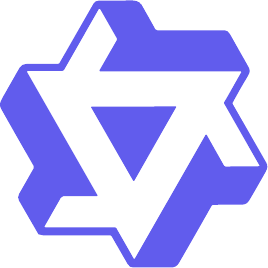}%
}

\begin{figure}[t]
    \centering

    \centering
    \includegraphics[width=\textwidth]{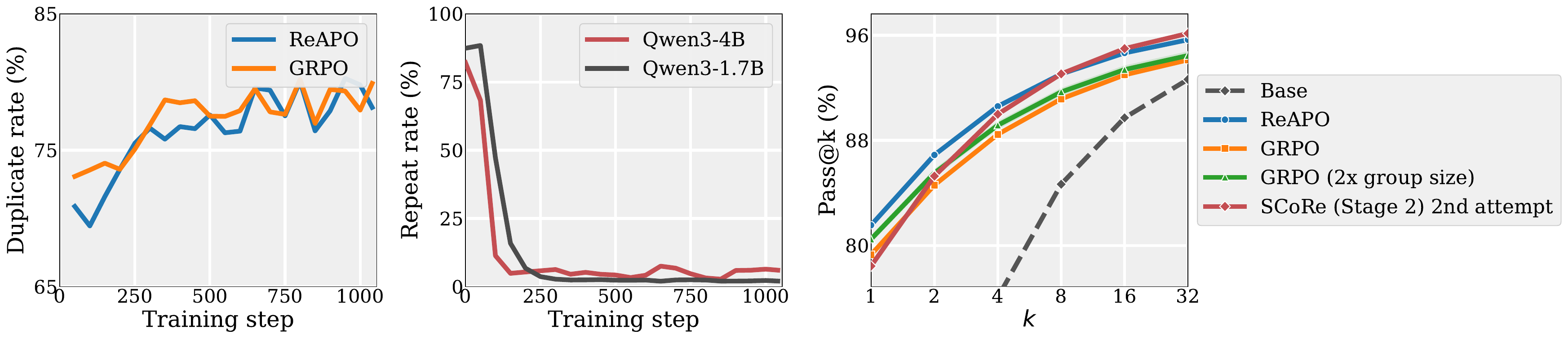}

    \caption{\textbf{Answer diversity, retry behavior, and solution coverage.} (\emph{Left}) Percentage of incorrect first-attempt answers that are duplicates within the same question group on AIME 2025 for Qwen3-4B-Base.
    (\emph{Middle}) The percentage of second-attempt answers that exactly match the failed first-attempt answer.
    (\emph{Right}) Pass@$k$ on GSM8K-Platinum for Qwen3-0.6B-Base. The SCoRe-Stage2 curve uses second responses generated regardless of first-attempt correctness; all other curves use first responses.}
    \label{fig:retry-repetition}
\end{figure}

\section{Discussion and Related Work}
\label{sec:conclusion}

\paragraph{Related work.}  ReAPO builds on prior work that has emphasized the importance of models reflecting on prior failed attempts during reasoning, such as self-correction, self-distillation, and multi-attempt reinforcement learning \cite{madaan2023selfrefineiterativerefinementselffeedback,kumar2024traininglanguagemodelsselfcorrect,Guo_2025,chung2025learningfailuresmultiattemptreinforcement,bensal2025reflectretryrewardselfimproving,sprague2026skillfactoryselfdistillationlearningcognitive}.  In particular, earlier works like \citet{kumar2024traininglanguagemodelsselfcorrect} use multi-turn RL to jointly optimize multiple attempts, but does not use verifier conditioning (leading to potentially inefficient exploration) and also relies on a complicated and somewhat bespoke reward shaping.  Other works like \citet{chung2025learningfailuresmultiattemptreinforcement,bensal2025reflectretryrewardselfimproving,sprague2026skillfactoryselfdistillationlearningcognitive} explore various aspects of multi-attempt reasoning and self-reflection, but typically do not integrate verifier-conditioned retries with attempt-level credit assignment as ReAPO does.  Unlike these other methods, ReAPO's credit assignment is motivated by considering the multi-attempt RL problem as a Stochastic Shortest Paths problem \citep{6dfe7762-0280-326e-88a5-877d93306e7b} and leveraging the rich understanding of such problems from classical RL to inform the design of our algorithm.  Further discussion can be found in \Cref{sec:related-work}.

\paragraph{Limitations and future work.}  We demonstrated the efficacy of ReAPO across a range of mathematical and scientific reasoning tasks, but resource constraints prohibited investigation into additional domains such as code or agentic settings.  We also leave further scaling of the retry budget $K$ with larger inference-time token budgets to future work, which may well offer further gains in reasoning accuracy.

\subsection*{Acknowledgments}
We gratefully acknowledge use of the research computing resources of the Empire AI Consortium, Inc, with support from Empire State Development of the State of New York, the Simons Foundation, and the Secunda Family Foundation.  We also thank the Yale Center for Research Computing  and the Texas Advanced Computing Center (TACC) at The University of Texas at Austin for providing additional computational resources.

\bibliography{refs}
\bibliographystyle{plainnat}

\newpage

\appendix

\clearpage
\section{Related Work}
\label{sec:related-work}

We now discuss in detail recent work related to ReAPO, with a special focus on alternative methods for structured exploration and multi-attempt reasoning in language models.

\paragraph{Knowledge distillation.}
Knowledge distillation provides dense token-level supervision by training a student policy to match the output distribution of a teacher model. On-policy distillation reduces the mismatch between training and inference by evaluating the teacher on trajectories sampled from the student itself \cite{agarwal2024onpolicydistillationlanguagemodels}. Subsequent work extends distillation to models with different tokenizers through token-alignment and approximate likelihood objectives \cite{boizard2025crosstokenizerdistillationuniversallogit,minixhofer2025universalcrosstokenizerdistillationapproximate}. Although these approaches provide substantially denser supervision than binary outcome rewards, they generally require access to a separate teacher model and its token-level predictions.

\paragraph{Self-distillation and self-revision.}
Self-distillation removes the external teacher by using the same model under different conditioning contexts as both teacher and student. Recent methods construct teacher policies using privileged reasoning traces, expert demonstrations, or rich environment feedback
\cite{zhao2026selfdistilledreasoneronpolicyselfdistillation,hubotter2026reinforcementlearningselfdistillation,shenfeld2026selfdistillationenablescontinuallearning}. Other approaches reorganize a model's own generations to teach cognitive behaviors such as verification, backtracking, and retrying, or train a reviser to transform binary rewards into dense token-level supervision
\cite{sprague2026skillfactoryselfdistillationlearningcognitive,he2026selfdistillationzeroselfrevisionturns}. These methods reduce dependence on a stronger teacher, but typically introduce a separate supervised or distillation objective and may require demonstrations, successful traces, or an additional revision-training stage. In contrast, ReAPO uses only verifier rewards and directly optimizes the policy through an attempt-level GRPO objective.

A further relevant line of work is the STaR (Self-Taught Reasoner) framework of \citet{zelikman2022star} and its descendants \citep{he2025star,li2025start,zelikman2024quiet}, which is an iterative self-training procedure for learning chain-of-thought reasoning.  The model generates an answer and, when this answer is incorrect, is prompted with the correct answer and asked for a rationalization for the correct answer.  The hint is then removed and the reasoning trace is kept as SFT data, before the process is iterated.  Like ReAPO, this approach leverages verifier-conditioned sampling, but the way in which it assigns `credit' (and indeed the entire framework around rationalization) is very different from ReAPO's attempt-level advantage assignment.

\paragraph{Stochastic shortest paths.}
Stochastic shortest-path problems are decision processes in which a policy seeks to reach a terminal state while minimizing the expected cost accumulated along the way \cite{6dfe7762-0280-326e-88a5-877d93306e7b}. In our setting, the state includes the problem and preceding failed attempts, a generated response either reaches the terminal state when verified correct or leads to a retry state, and each failed response incurs a unit cost. This gives the idealized objective $\mathbb{E}_{\pi_\theta}[\tau-1]$, the expected number of failures before the first success. In general, the stochastic shortest-path formulation also accounts for how an attempt affects future retry states. The simplifying assumptions in \Cref{sec:method} let us use this perspective to motivate attempt-level credit assignment; our finite-budget, GRPO-style update is the practical algorithm built from that motivation.

\paragraph{Multi-attempt reinforcement learning.}
SCoRe \cite{kumar2024traininglanguagemodelsselfcorrect} trains a single language model to intrinsically correct its own responses using on-policy, multi-turn reinforcement learning. Its first stage constrains the first-attempt policy to remain close to the base model while optimizing the second attempt. Its second stage jointly optimizes both attempts and augments the second-attempt reward with a progress bonus proportional to $r_2-r_1$. This procedure discourages behavior collapse and teaches the model to make productive revisions. SCoRe also shows that, for a fixed number of generated solutions, combining independent sampling with sequential self-correction can outperform fully parallel sampling at inference time.

ReAPO differs from SCoRe in both its feedback setting and optimization objective. SCoRe produces a second response regardless of the correctness of the first and does not reveal the first response's verifier outcome in its revision instruction. ReAPO generates a retry only following a verifier-confirmed failure, making the additional attempt explicitly conditioned on evidence that the previous solution failed. Rather than using a separate initialization stage and a manually weighted progress bonus, ReAPO pools the rewards of all generated attempts and assigns each attempt's resulting group-relative advantage only to its own response tokens.

~\cite{chung2025learningfailuresmultiattemptreinforcement} formulate reasoning as a multi-attempt task and optimize the resulting
interaction using multi-turn reinforcement learning. However, their method does not construct separate group-relative advantages for different attempt stages. ~\cite{bensal2025reflectretryrewardselfimproving} generate a self-reflection after an incorrect response and reward the reflection tokens when the subsequent retry succeeds. This trains the model to produce useful reflections but does not directly assign separate advantages to the answer tokens of each attempt.

Closest to our work is the recent, independent work of \citet{ildiz2026learning}, which also studies verifier-conditioned multi-attempt reasoning and assigns credit to individual attempts. Their calibrated method optimizes Ver@$K$, the probability that at least one of the first $K$ attempts succeeds. ReAPO is instead motivated by minimizing the expected number of failed attempts before success and primarily evaluates first-response performance. This distinction matters when only the first response is scored or reliable verifier feedback is unavailable at inference. On MATH, CAL-GRPO's Ver@2 advantage over naive attempt-level GRPO comes primarily from stronger correction after a failed first attempt, while the naive variant attains higher first-attempt success. Their MATH comparison includes trajectory-level, naive attempt-level, and calibrated GRPO on multi-attempt rollouts, but no independent single-attempt GRPO control; our experiments include standard GRPO trained and evaluated on the same models and benchmarks as ReAPO. More concretely, CAL-GRPO normalizes rewards within each attempt stage and weights attempts by the estimated probability that later attempts fail, whereas ReAPO pools rewards across generated attempts and assigns each advantage only to its corresponding response tokens (\Cref{sec:method}).

Together, these methods show that conditioning later generations on earlier attempts can improve reasoning. ReAPO specifically combines retries triggered by verifier-confirmed failures with pooled attempt-level normalization, providing direct contrast between successful and unsuccessful attempted responses.

\clearpage
\section{Discussion and Future Work}
\label{sec:discussion}

\paragraph{Normalization across attempts.} Standard GRPO normalizes rewards among responses generated from the same prompt, whereas ReAPO pools initial responses with retries that are additionally conditioned on a failed response and verifier feedback. The attempts in the expanded group therefore do not share identical conditioning contexts. For open-ended reasoning problems, where many distinct trajectories can lead to either correct or incorrect answers, observing one failed solution may provide relatively little information about the correct solution. Nevertheless, this assumption may become problematic in restricted answer spaces. For example, in multiple-choice or true--false tasks, revealing that an answer is incorrect can substantially narrow the remaining possibilities without requiring improved reasoning. In such settings, gains could partly reflect answer elimination rather than structured exploration. Future work could study how this affects performance and possible ways to account for this.

\paragraph{Diminishing returns from additional attempts.} Sequential retries introduce additional generation cost and training latency. Moreover, each successive attempt is generated only for the subset of trajectories that failed every preceding attempt, so as early attempts solve easy questions, later attempts are concentrated on increasingly difficult problems and may have a low probability of success. This behavior is already visible in our three-attempt ablation on the 0.6B model: first-attempt accuracy exceeds $80\%$, cumulative accuracy after the second attempt exceeds $90\%$, and consequently fewer than $10\%$ of trajectories remain eligible for a third attempt. The third attempt therefore provides little additional training signal and yields no clear improvement over the two-attempt configuration. Future work could study how this method scales as the number of attempts increases.

\paragraph{Token allocation and truncation.} ReAPO allocates a fixed cumulative response-token budget to each trajectory, shared across its attempts. This can be inefficient when an early attempt consumes a large number of tokens before failing, particularly when the complete interaction is constrained by a shared token budget. In the Qwen3-4B-Base experiments, more than $50\%$ of ReAPO trajectories reach the maximum generation length near the end of training, compared with under $10\%$ for GRPO (Figure~\ref{fig:training-dynamics}). This rise in truncation coincides with ReAPO's sharper decline in validation performance after approximately 600 steps; GRPO also declines, but without a comparable rise in truncation. A separate token budget for each attempt, or an adaptive allocation that reserves sufficient capacity for retries, could allow later attempts to complete their reasoning even when earlier attempts are long.

\paragraph{Verifier reliability and task coverage.} ReAPO assumes that the verifier can reliably identify an incorrect response. A false-negative verification causes the model to receive misleading failure feedback and may incorrectly penalize a valid solution, while a false positive prevents a potentially useful retry. Our experiments focus primarily on mathematical and scientific reasoning tasks with automatically checkable answers. Evaluating robustness to noisy or learned verifiers, and extending the method to code generation, tool use, and domains with less reliable outcome feedback, are important directions for future work.

\paragraph{Compatibility with other policy-optimization methods.} Verifier-conditioned rollout collection and attempt-level credit assignment are not specific to the standard GRPO objective. The same construction could be combined with alternative normalization and optimization methods, including Dr.GRPO ~\cite{liu2025understandingr1zeroliketrainingcritical}, DAPO~\cite{yu2025dapo}, and MaxRL~\cite{tajwar2026maximumlikelihoodreinforcementlearning}. Studying these combinations could determine whether retry-aware exploration provides complementary improvements when paired with objectives designed to address other weaknesses of GRPO.

\clearpage
\section{Experimental Details}
\label{app:experimental-details}

Except for the training of SCoRe Stage-II baseline, we adopt the training and evaluation hyperparameters used by MaxRL \cite{tajwar2026maximumlikelihoodreinforcementlearning}.

\subsection{Qwen3-0.6B Experiments}
\label{app:qwen06}

The Qwen3-0.6B-Base was trained \cite{yang2025qwen3technicalreport} on the GSM8K training set \cite{cobbe2021trainingverifierssolvemath} and evaluated on GSM8K-Platinum \cite{vendrow2025largelanguagemodelbenchmarks}. The string ``Let's think step by step and output the final answer within \texttt{\textbackslash boxed\{\}}.'' was appended to each problem and format the resulting
conversation using the Qwen3 tokenizer chat template. If the first response is incorrect, the retry instruction is appended as a new user message. An example is shown below.

\begin{figure}[H]
\begin{tcolorbox}[promptbox,title={Qwen3-0.6B and OLMo-2-0425-1B-SFT Chat Template}]
\chatline{<|im_start|>user}
\chatline{Natalia sold clips to 48 of her friends in April,}
\chatline{and then she sold half as many clips in May. How}
\chatline{many clips did Natalia sell altogether in April}
\chatline{and May?}
\medskip
\chatline{Let's think step by step and output the final}
\chatline{answer within \boxed{}.}
\chatline{<|im_end|>}
\chatline{<|im_start|>assistant}
\chatline{{first attempt}}
\chatline{<|im_end|>}
\chatline{<|im_start|>user}
\chatline{Your answer is wrong. Reflect on your previous}
\chatline{answer and try something else.}
\chatline{<|im_end|>}
\chatline{<|im_start|>assistant}
\chatline{{retry}}
\chatline{<|im_end|>}
\end{tcolorbox}
\end{figure}

We implement all experiments using VERL \cite{sheng2024hybridflow} and train each model on four NVIDIA B200 GPUs provided by Empire AI \cite{Bloom2025EmpireAI}. Training is fully on-policy, so we do not use importance ratios or associated clipping. For all Qwen3-0.6B runs except the SCoRe Stage-II-only baseline, we disable the KL penalty and entropy bonus. The SCoRe baseline starts from Qwen3-0.6B-Base without Stage I training and follows the Stage II objective and SCoRe-specific hyperparameters reported for MATH by \citet{kumar2024traininglanguagemodelsselfcorrect}, including the progress-bonus weight $\alpha=10$ and reference-policy KL weight $\beta_1=0.01$. Training rollouts are sampled at temperature $1.0$, while validation uses temperature $0.6$, top-$p=0.95$, and no top-$k$ restriction.

\begin{table*}[t]
\centering
\small
\caption{\textbf{Qwen3-0.6B-Base training hyperparameters.} Results are aggregated across eight training seeds. The maximum-attempt setting $K$ applies to ReAPO; GRPO uses one attempt. The KL is 0 for every experiment besides SCoRe. These default values are used for all 0.6B experiments unless stated otherwise.}
\vspace{8pt}
\setlength{\tabcolsep}{5pt}
\renewcommand{\arraystretch}{1.08}
\begin{tabularx}{\textwidth}{@{}lYlY@{}}
\toprule
\textbf{Hyperparameter} & \textbf{Value} & \textbf{Hyperparameter} & \textbf{Value} \\
\midrule
Training dataset & GSM8K & Training hardware & $4\times$ NVIDIA B200 \\
Group size & $16$ & Maximum attempts & $2$ \\
Maximum prompt length & $1{,}024$ & Maximum response length & $2{,}048$ \\
Training batch size & $64$ & Actor learning rate & $1\times10^{-6}$ \\
Learning-rate warmup & $0$ steps & Gradient clipping & $0.3$ \\
KL-loss coefficient & $0$ & Entropy coefficient & $0$ \\
Training temperature & $1.0$ & Validation samples & $32$ per prompt \\
Validation temperature & $0.6$ & Validation top-$p$ & $0.95$ \\
Validation top-$k$ & None & Step & $580$ \\
\bottomrule
\end{tabularx}
\label{tab:qwen06-hyperparameters}
\end{table*}

\subsection{Qwen3-1.7B-Base and Qwen3-4B-Base Experiments}
\label{app:qwen-large}

We train Qwen3-1.7B-Base and Qwen3-4B-Base \cite{yang2025qwen3technicalreport} on Polaris-53K \cite{Polaris2025}. We use the answer provided by Polaris-53K as the reference answer and use \texttt{math\_verify} to assign a binary correctness reward to each attempt separately.

We append ``Please reason step by step, and put your final answer within \texttt{\textbackslash boxed\{\}}.'' to each problem and format the resulting conversation using the Qwen3 tokenizer chat template. If the first response is incorrect, the retry instruction is appended as a new user message. An example is shown below. For evaluation, we sample 256 responses per problem for BeyondAIME, AIME 2024, AIME 2025, and HMMT 2025, and 64 responses per problem for MATH-500, Minerva, and GPQA-Diamond. At the selected checkpoint, we run a separate high-sample evaluation for Pass@$k$ and report its corresponding first-attempt mean accuracy in Table~\ref{tab:large-model-results}.

\begin{figure}[H]
\begin{tcolorbox}[promptbox,title={Qwen3-1.7B/4B-Base Chat Template}]
\chatline{<|im_start|>user}
\chatline{All the complex roots of (z + 1)^4 = 16z^4,}
\chatline{when plotted in the complex plane, lie on a}
\chatline{circle. Find the radius of this circle.}
\medskip
\chatline{Please reason step by step, and put your final}
\chatline{answer within \boxed{}.}
\chatline{<|im_end|>}
\chatline{<|im_start|>assistant}
\chatline{{first attempt}}
\chatline{<|im_end|>}
\chatline{<|im_start|>user}
\chatline{Your answer is wrong. Reflect on your previous}
\chatline{answer and try something else.}
\chatline{<|im_end|>}
\chatline{<|im_start|>assistant}
\chatline{{retry}}
\chatline{<|im_end|>}
\end{tcolorbox}
\end{figure}

We implement all experiments using VERL \cite{sheng2024hybridflow} and train each model on four NVIDIA B200 GPUs provided by Empire AI \cite{Bloom2025EmpireAI}. We use the same training configuration for Qwen3-1.7B-Base and Qwen3-4B-Base. Training is fully on-policy, so we do not use importance ratios or associated clipping. We also disable the KL penalty and entropy bonus. Training rollouts are sampled at temperature $1.0$, while validation uses temperature $0.6$, top-$p=0.95$, and no top-$k$ restriction. 

\begin{table*}[t]
\centering
\small
\caption{\textbf{Qwen3-1.7B-Base and Qwen3-4B-Base training hyperparameters.} Results are aggregated across seven seeds for Qwen3-1.7B-Base and three seeds for Qwen3-4B-Base. The maximum-attempt setting applies to ReAPO; GRPO uses one attempt.}
\vspace{8pt}
\setlength{\tabcolsep}{5pt}
\renewcommand{\arraystretch}{1.08}
\begin{tabularx}{\textwidth}{@{}lYlY@{}}
\toprule
\textbf{Hyperparameter} & \textbf{Value} & \textbf{Hyperparameter} & \textbf{Value} \\
\midrule
Training dataset & Polaris-53K & Training hardware & $4\times$ NVIDIA B200 \\
Group size & $16$ & Maximum attempts & $2$ \\
Maximum prompt length & $1{,}024$ & Maximum response length & $4{,}096$ \\
Training batch size & $256$ & Actor learning rate & $1\times10^{-6}$ \\
Learning-rate warmup & $0$ steps & Gradient clipping & $0.3$ \\
KL-loss coefficient & $0$ & Entropy coefficient & $0$ \\
Training temperature & $1.0$ & Validation samples & $32$ per prompt \\
Validation temperature & $0.6$ & Validation top-$p$ & $0.95$ \\
Validation top-$k$ & None & Step & $1{,}040$ \\
\bottomrule
\end{tabularx}
\label{tab:qwen-large-hyperparameters}
\end{table*}

\subsection{Ablation Setup}
\label{app:ablation-setup}

\paragraph{Number of retry attempts.}
Our default ReAPO configuration permits at most $K=2$ attempts per trajectory. To study the effect of a larger retry budget, we additionally train a variant with $K=3$, allowing a second retry when the first two attempts are incorrect.

\paragraph{Retry prompt.}
The default retry instruction is ``Your answer is wrong. Reflect on your previous answer and try something else.'' We compare this against a shorter prompt, ``Try again,'' while keeping all other training settings fixed.

\paragraph{Attempt-level credit assignment.}
To isolate the effect of ReAPO's credit assignment, we consider a multi-turn baseline that uses the same verifier-conditioned retry process but treats the entire sequence of attempts as a single rollout. The model continues attempting the problem until it produces a correct answer or reaches the retry limit, after which the complete trajectory receives a single binary reward and one trajectory-level advantage is applied to all tokens in the sequence.

\clearpage
\section{Additional Results}
\label{app:additional-results}

\subsection{Per-Benchmark Training Curves}

Figures~\ref{fig:qwen17-training-curves} and \ref{fig:qwen4-training-curves} report first-attempt Pass@1 and Pass@32 throughout training. For Qwen3-1.7B-Base, ReAPO generally improves more steadily than GRPO and finishes ahead on most benchmarks for both metrics. The difference is particularly consistent on BeyondAIME, MATH-500, and Minerva. The Pass@32 results indicate that the improvement in first-attempt accuracy does not come from collapsing onto a narrower set of solutions.

\begin{figure}[t]
    \centering
    \includegraphics[width=\textwidth]{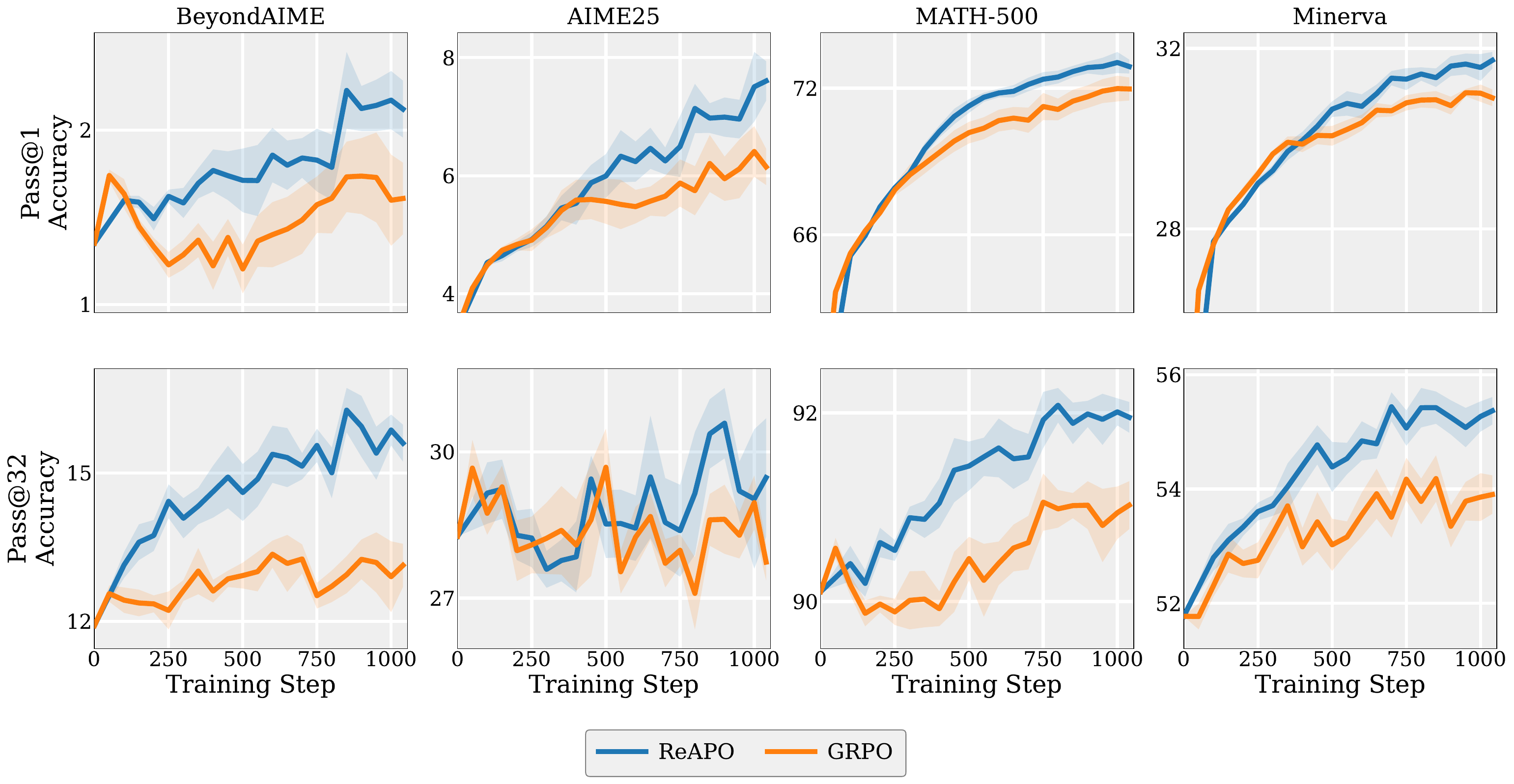}
    \caption{\textbf{Per-benchmark training curves for Qwen3-1.7B-Base.} The top row reports first-attempt Pass@1 and the bottom row reports Pass@32. Curves are averages across seven training seeds and shaded regions denote standard errors.}
    \label{fig:qwen17-training-curves}
\end{figure}

For Qwen3-4B-Base, ReAPO reaches stronger performance than GRPO on most benchmarks during the first 500--600 training steps. After this point, however, the performance of both methods declines on several tasks. Because the 4B experiments use the same optimization hyperparameters as the 1.7B experiments, this shared degradation suggests that these settings may be overly aggressive or otherwise insufficiently stable for the larger model. We therefore report results at step 500, rather than at the final training step. During training, we selected checkpoints using MATH-500 validation with 32 responses per prompt (Table \ref{tab:qwen-large-hyperparameters}), before running the larger per-step evaluations. Both ReAPO and GRPO peaked near step 500 under this validation procedure, so we report the step-500 checkpoint rather than the final checkpoint. However, model-scale-specific tuning of the learning rate and other optimization parameters may further improve both methods.

\begin{figure}[t]
    \centering
    \includegraphics[width=\textwidth]{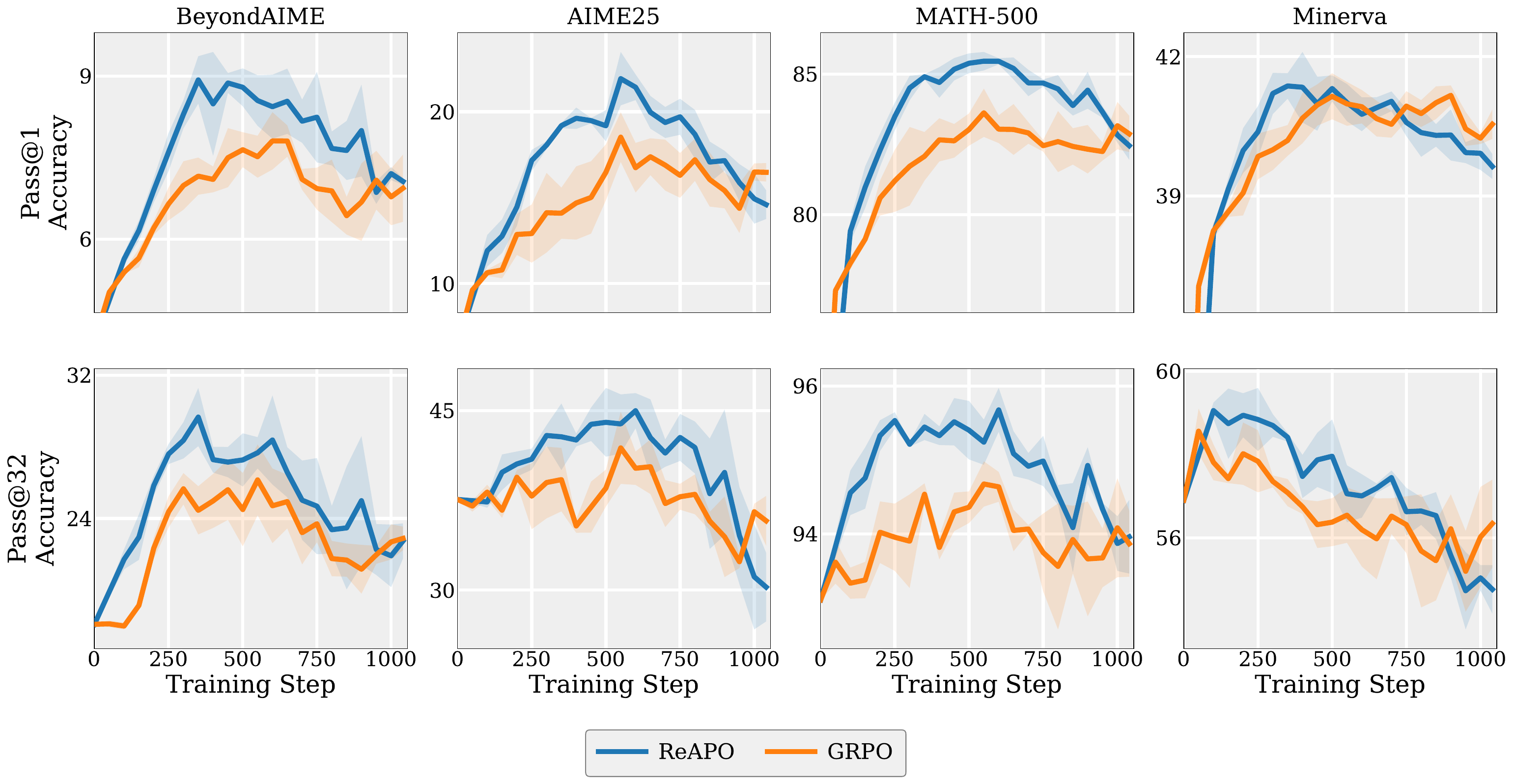}
    \caption{\textbf{Per-benchmark training curves for Qwen3-4B-Base.} The top row reports first-attempt Pass@1 and the bottom row reports Pass@32. Curves are averages across three training seeds and shaded regions denote standard errors.}
    \label{fig:qwen4-training-curves}
\end{figure}

\begin{figure}[t]
    \centering
    \includegraphics[width=\textwidth]{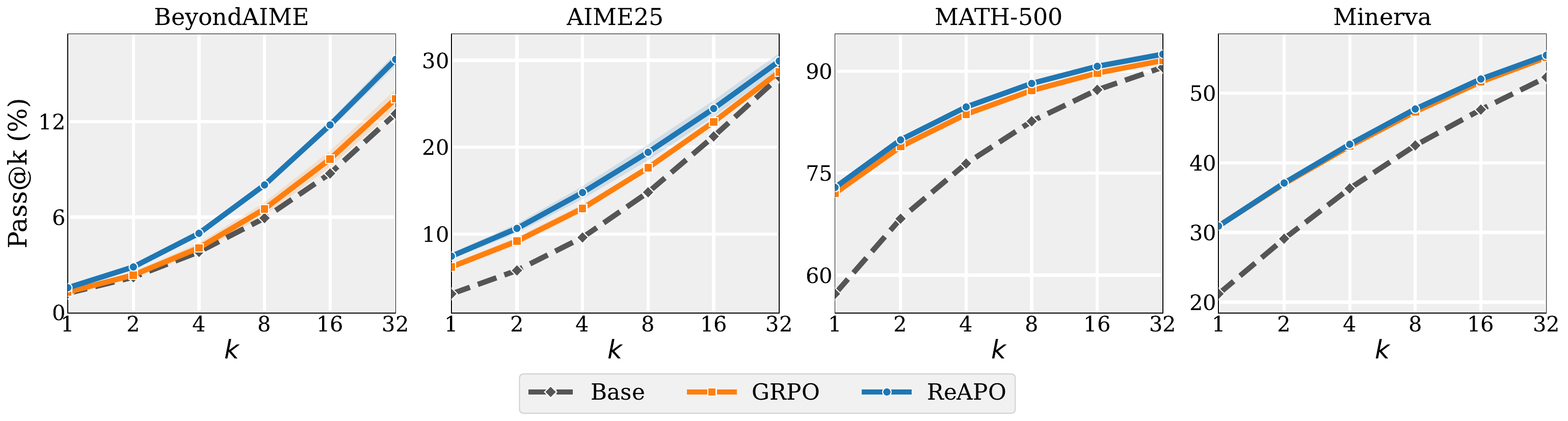}
    \caption{\textbf{First-attempt Pass@$k$ for Qwen3-1.7B-Base} at the last checkpoint. GRPO and ReAPO curves are means across seven training seeds, with shaded regions denoting standard errors.}
    \label{fig:qwen_1.7b_best32}
\end{figure}

\begin{figure}[t]
    \centering
    \includegraphics[width=\textwidth]{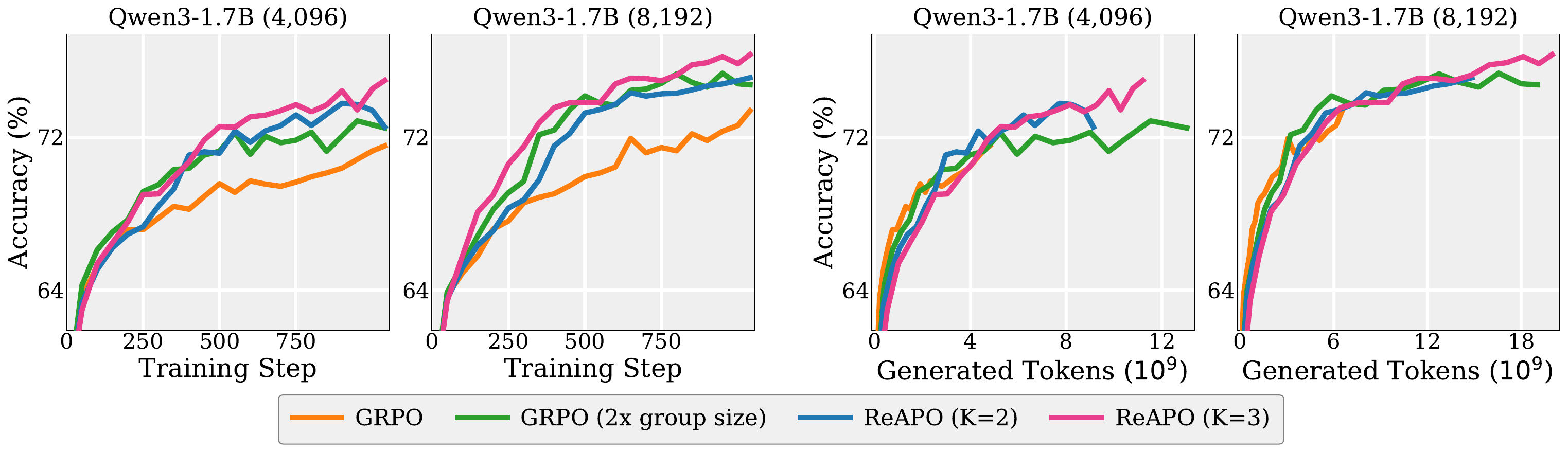}
    \caption{\textbf{Attempt budget, group size, and response length on Qwen3-1.7B-Base (MATH-500) using one training seed.} First-attempt accuracy against optimization steps (left two panels) and cumulative generated response tokens (right two panels) for maximum response lengths of 4,096 and 8,192 tokens per trajectory.}
    \label{fig:qwen17-ablation}
\end{figure}

\begin{table}[t]
\centering
\small
\caption{\textbf{Qwen3-1.7B-Base, seed 42, with different attempt budgets, group sizes, and maximum response lengths.} First-attempt accuracy (\%) from the high-sample evaluation at the selected checkpoint (as in Table~\ref{tab:large-model-results}). Single seed; the best value in each response-length group is bold.}
\vspace{8pt}
\setlength{\tabcolsep}{2.5pt}
\renewcommand{\arraystretch}{1.08}
\scalebox{0.92}{%
\begin{tabular}{@{}l*{7}{c}@{\hspace{8pt}}c@{}}
\toprule
\textbf{Model / Method} & \textbf{\shortstack{Beyond\\AIME}} & \textbf{\shortstack{AIME\\2024}} & \textbf{\shortstack{AIME\\2025}} & \textbf{\shortstack{HMMT\\2025}} & \textbf{\shortstack{MATH\\500}} & \textbf{\shortstack{Minerva\\~}} & \textbf{\shortstack{GPQA\\Diamond}} & \textbf{\shortstack{Avg.\\~}} \\
\midrule
Qwen3-1.7B & $1.2$ & $4.2$ & $3.1$ & $0.2$ & $57.2$ & $21.9$ & $10.8$ & $14.1$ \\
\midrule
\multicolumn{9}{@{}l}{\textit{Max response length 4,096}} \\
GRPO & $0.8$ & $10.4$ & $6.2$ & $1.0$ & $71.2$ & $30.5$ & $21.2$ & $20.2$ \\
GRPO ($2\times$ group size) & $1.8$ & $12.6$ & $7.7$ & $\mathbf{1.8}$ & $73.1$ & $\mathbf{32.5}$ & $25.0$ & $\mathbf{22.1}$ \\
\textbf{ReAPO} ($K{=}2$) & $1.4$ & $\mathbf{13.9}$ & $7.3$ & $1.1$ & $72.6$ & $32.0$ & $\mathbf{26.2}$ & $\mathbf{22.1}$ \\
\textbf{ReAPO} ($K{=}3$) & $\mathbf{1.9}$ & $13.5$ & $\mathbf{8.0}$ & $1.4$ & $\mathbf{74.7}$ & $32.0$ & $22.2$ & $22.0$ \\
\midrule
\multicolumn{9}{@{}l}{\textit{Max response length 8,192}} \\
GRPO & $1.3$ & $12.9$ & $7.6$ & $1.9$ & $72.9$ & $31.2$ & $25.8$ & $21.9$ \\
GRPO ($2\times$ group size) & $1.4$ & $12.9$ & $9.0$ & $2.8$ & $74.7$ & $32.4$ & $26.2$ & $22.8$ \\
\textbf{ReAPO} ($K{=}2$) & $\mathbf{1.9}$ & $13.9$ & $\mathbf{11.8}$ & $2.9$ & $75.0$ & $32.0$ & $\mathbf{26.5}$ & $23.4$ \\
\textbf{ReAPO} ($K{=}3$) & $\mathbf{1.9}$ & $\mathbf{15.7}$ & $9.4$ & $\mathbf{3.7}$ & $\mathbf{76.3}$ & $\mathbf{33.0}$ & $24.6$ & $\mathbf{23.5}$ \\
\bottomrule
\end{tabular}%
}
\label{tab:qwen17b-seed42-eval}
\end{table}

\begin{table}[t]
\centering
\small
\caption{\textbf{Qwen3-0.6B-Base results on GSM8K-Platinum at step 580.} All columns except SCoRe-Stage2 report first-response accuracy; SCoRe-Stage2 reports second-response accuracy without verifier feedback at inference. Values are mean $\pm$ standard error across eight seeds.}
\vspace{8pt}
\scalebox{0.92}{%
\setlength{\tabcolsep}{3pt}
\renewcommand{\arraystretch}{1.04}
\begin{tabular}{@{}l*{7}{c}@{}}
\toprule
\textbf{Metric} & \shortstack{Qwen3-0.6B\\Base} & GRPO & \shortstack{GRPO\\(2$\times$ Rollouts)} & MaxRL & \shortstack{SCoRe\\(Stage 2)} & \shortstack{\textbf{ReAPO}\\($K{=}2$)} & \shortstack{\textbf{ReAPO}\\($K{=}3$)} \\
\midrule
Mean@32 & $44.3 \pm 0.1$ & $79.3 \pm 0.2$ & $80.5 \pm 0.3$ & $77.8 \pm 0.1$ & $78.4 \pm 0.1$ & $\mathbf{81.5 \pm 0.2}$ & $\mathbf{81.2 \pm 0.3}$ \\
Pass@32 & $92.6 \pm 0.1$ & $94.1 \pm 0.2$ & $94.5 \pm 0.3$ & $95.3 \pm 0.1$ & $\mathbf{96.1 \pm 0.1}$ & $95.7 \pm 0.2$ & $95.3 \pm 0.2$ \\
\bottomrule
\end{tabular}
}

\label{tab:qwen06-results}
\end{table}

\begin{figure}[!t]
    \centering
    \includegraphics[width=0.4\textwidth]{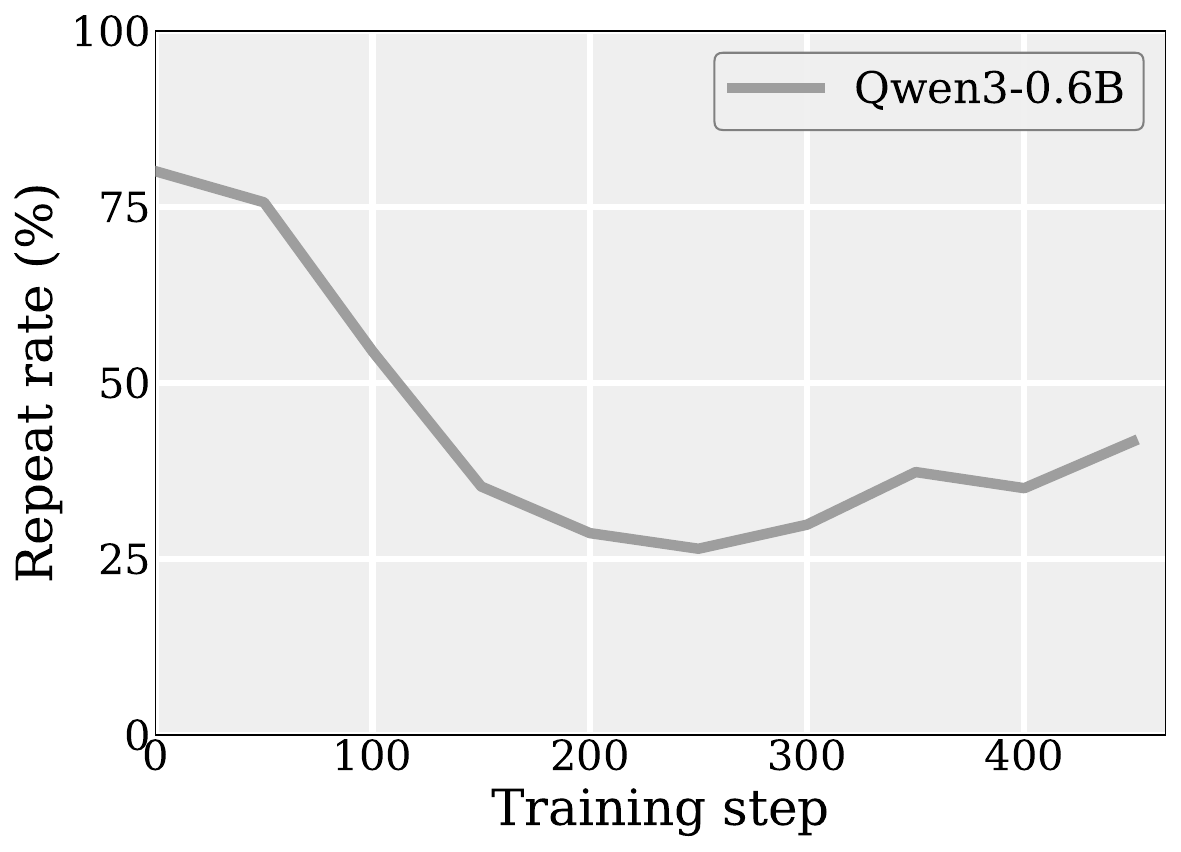}
    \caption{Percentage of second answers exactly matching the failed first answer on AIME 2025, for one seed. Exact repetition decreases less for the smaller model than for the larger models. Larger models seem to generate less repetitively than smaller ones.}
    \label{fig:qwen_0.6b_repetition}
\end{figure}

\subsection{Solution Coverage}

Table~\ref{tab:pass32-results} reports Pass@32 at the selected checkpoints. ReAPO numerically exceeds GRPO on all seven benchmarks at both model scales. For Qwen3-1.7B-Base, the improvement is statistically significant under the reported standard errors on six benchmarks, with the largest gain occurring on GPQA-Diamond. For Qwen3-4B-Base, ReAPO produces particularly large gains on AIME 2024, HMMT 2025, and GPQA-Diamond. However, reinforcement learning does not uniformly improve coverage over the base model: on GPQA-Diamond, the Qwen3-4B-Base base model retains higher Pass@32 than both trained policies.

\begin{table*}[t]
\centering
\small
\caption{\textbf{Pass@32 results at the selected checkpoints.} Values are percentages. GRPO and ReAPO are reported as mean $\pm$ standard error across training seeds, while base-model results use a single evaluation. A result is bold when its uncertainty interval lies strictly above those of the competing methods.}
\vspace{8pt}
\setlength{\tabcolsep}{2.5pt}
\renewcommand{\arraystretch}{1.08}
\scalebox{0.92}{%
\begin{tabular}{@{}l*{7}{c}@{}}
\toprule
\textbf{Model / Method} &
\textbf{\shortstack{Beyond\\AIME}} &
\textbf{\shortstack{AIME\\2024}} &
\textbf{\shortstack{AIME\\2025}} &
\textbf{\shortstack{HMMT\\2025}} &
\textbf{\shortstack{MATH\\500}} &
\textbf{\shortstack{Minerva\\~}} &
\textbf{\shortstack{GPQA\\Diamond}} \\
\midrule
Qwen3-1.7B-Base & $12.5$ & $26.9$ & $28.1$ & $5.5$ & $90.6$ & $52.9$ & $70.6$ \\
GRPO & $13.4 \pm 0.6$ & $30.8 \pm 0.3$ & $28.7 \pm 0.3$ & $12.2 \pm 1.0$ & $91.6 \pm 0.3$ & $55.1 \pm 0.3$ & $70.4 \pm 2.6$ \\
\textbf{ReAPO} (ours) & $\mathbf{15.9 \pm 0.4}$ & $31.1 \pm 0.6$ & $\mathbf{29.9 \pm 0.9}$ & $\mathbf{14.5 \pm 0.9}$ & $\mathbf{92.5 \pm 0.2}$ & $\mathbf{56.3 \pm 0.2}$ & $\mathbf{78.5 \pm 1.0}$ \\
\midrule
Qwen3-4B-Base & $18.2$ & $38.5$ & $37.9$ & $17.0$ & $94.2$ & $58.5$ & $\mathbf{73.5}$ \\
GRPO & $25.1 \pm 1.8$ & $42.8 \pm 0.9$ & $39.3 \pm 1.5$ & $20.5 \pm 0.9$ & $94.8 \pm 0.1$ & $57.1 \pm 0.5$ & $59.5 \pm 2.3$ \\
\textbf{ReAPO} (ours) & $28.1 \pm 1.3$ & $\mathbf{48.0 \pm 1.7}$ & $42.9 \pm 2.5$ & $\mathbf{24.5 \pm 0.9}$ & $\mathbf{95.6 \pm 0.4}$ & $58.5 \pm 0.9$ & $69.5 \pm 1.5$ \\
\bottomrule
\end{tabular}%
}
\label{tab:pass32-results}
\end{table*}

\subsection{Token-Normalized Performance}

Because ReAPO generates retries following incorrect responses, it generally produces more response tokens per optimization step than GRPO. Figure~\ref{fig:token-normalized-large-models} therefore compares the methods as a function of cumulative generated response tokens. Within the range in which both methods have observations, ReAPO remains competitive or stronger on most benchmarks. The clearest improvements occur on AIME 2025 and MATH-500 for the 4B model, while differences on Minerva are smaller. Thus, the improvements cannot be explained solely by ReAPO performing more generation within each optimization step.

\begin{figure}[t]
    \centering
    \includegraphics[width=\textwidth]{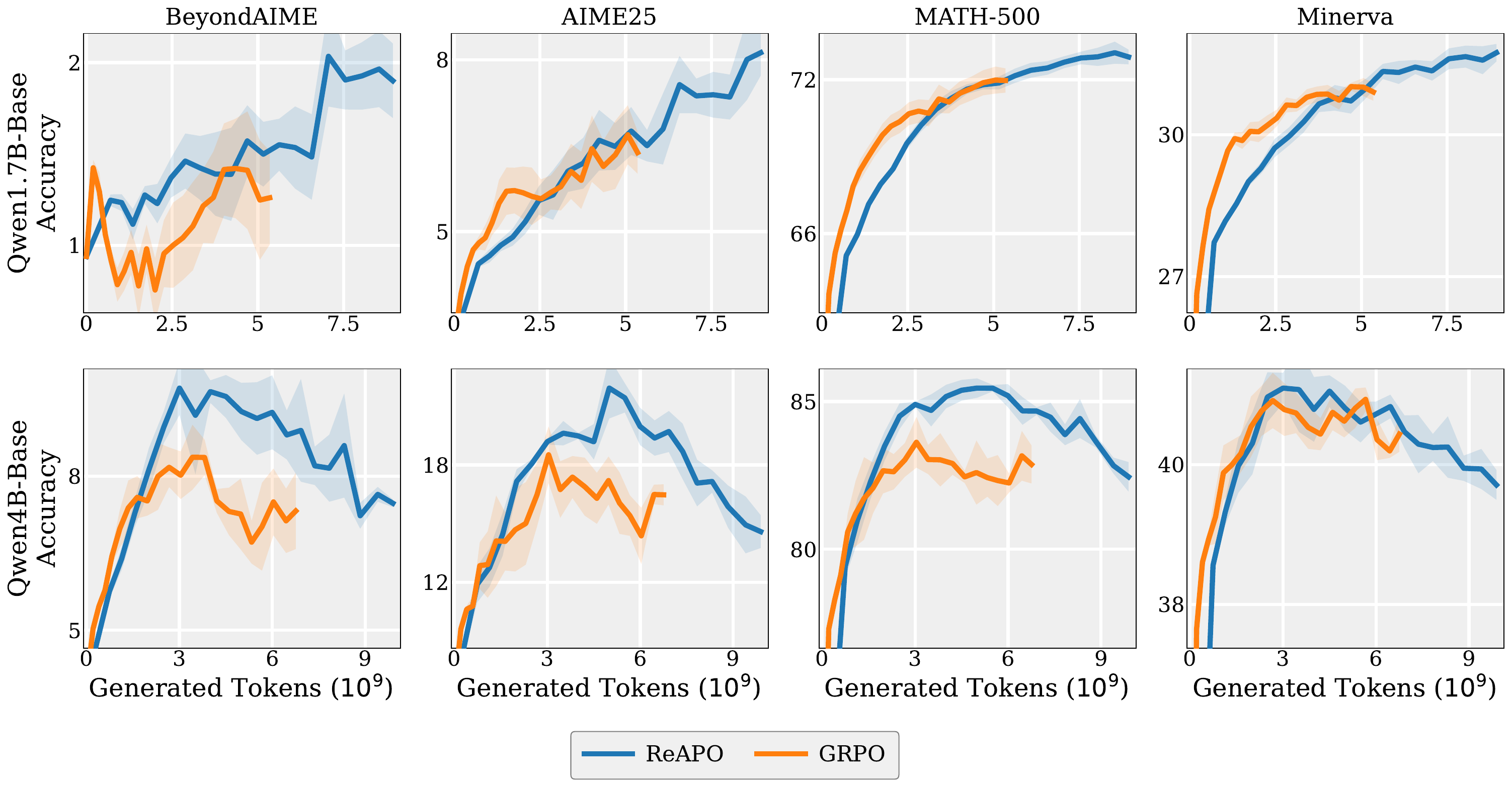}
    \caption{\textbf{Token-normalized first-attempt accuracy.} Performance is plotted against cumulative generated response tokens for Qwen3-1.7B-Base (top) and Qwen3-4B-Base (bottom). Curves show means across seven and three seeds, respectively, and shaded regions denote standard errors. Within the range of overlapping token counts, ReAPO is competitive with or outperforms GRPO on most benchmarks.}
    \label{fig:token-normalized-large-models}
\end{figure}

\subsection{Training Dynamics}

\begin{figure}[!t]
    \centering
    \includegraphics[width=\textwidth]{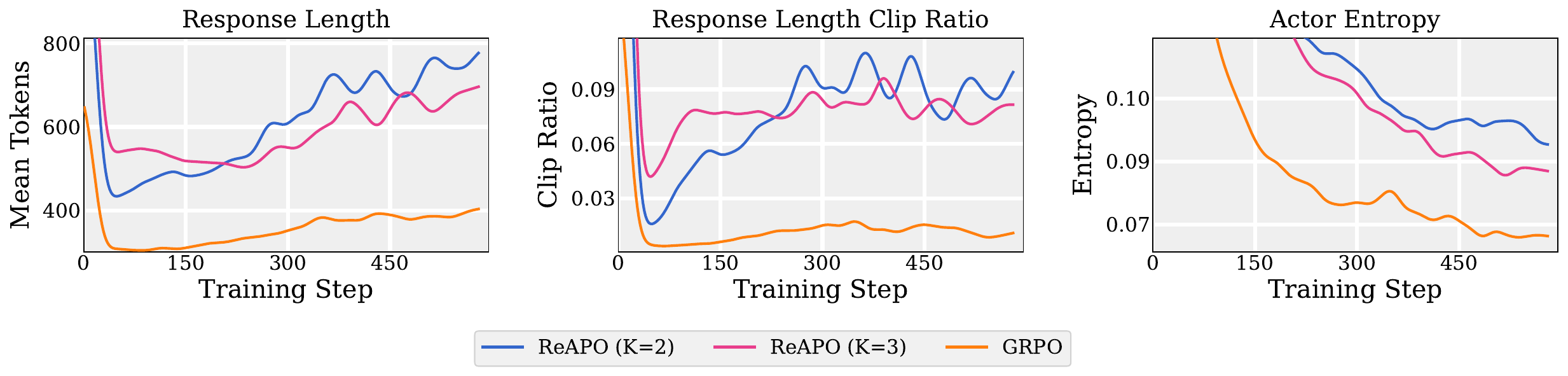}
    \vspace{3mm}

    \includegraphics[width=\textwidth]{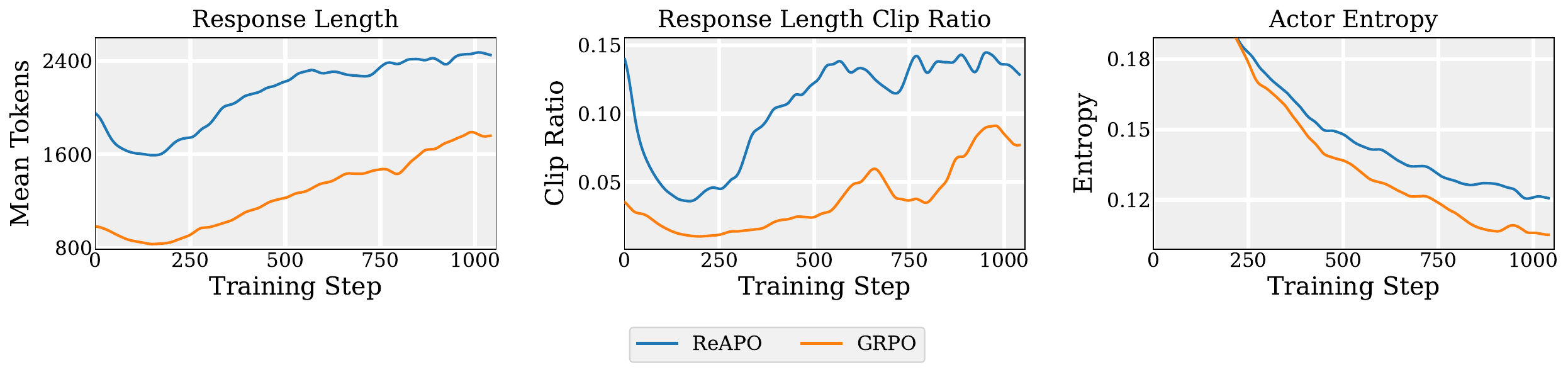}
    \vspace{3mm}

    \includegraphics[width=\textwidth]{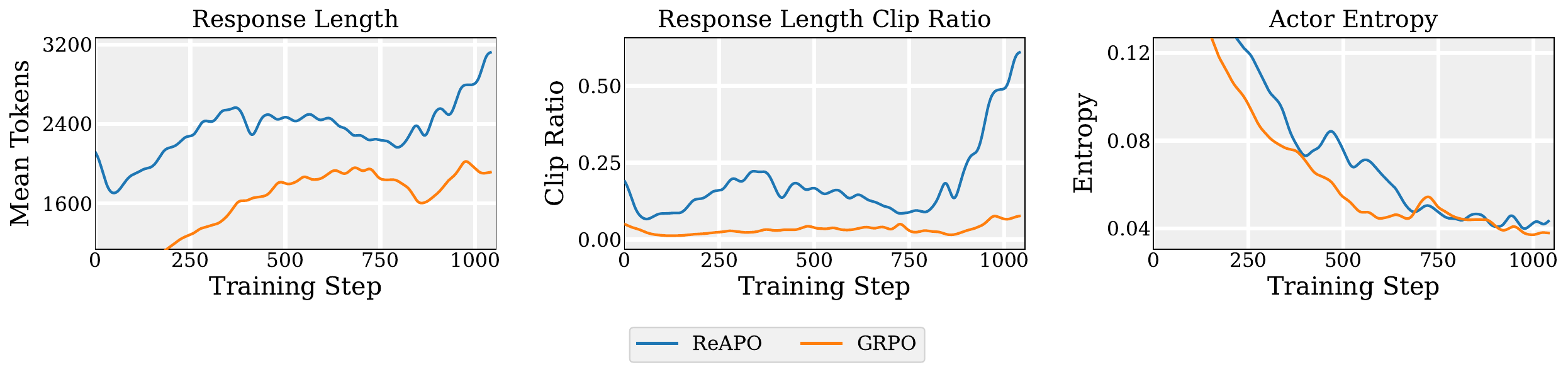}
    \caption{\textbf{Training dynamics for Qwen3-0.6B-Base, Qwen3-1.7B-Base, and Qwen3-4B-Base.} From top to bottom, the rows report results for Qwen3-0.6B-Base, Qwen3-1.7B-Base, and Qwen3-4B-Base, averaged across eight, seven, and three seeds, respectively.}
    \label{fig:training-dynamics}
\end{figure}

For Qwen3-1.7B-Base and Qwen3-0.6B-Base, ReAPO maintains higher actor entropy than GRPO throughout most of training (shown in Figure~\ref{fig:training-dynamics}), indicating a less concentrated token distribution. Higher entropy is not inherently better, but its persistence is consistent with ReAPO retaining more stochasticity while improving validation accuracy. For Qwen3-4B-Base, the sharp late increase in response length and clipping coincides with the decline in validation performance observed in Figure~\ref{fig:qwen4-training-curves}. This suggests that the later checkpoints experience increasing truncation pressure, although these diagnostics alone do not establish a causal relationship.

\begin{table*}[t]
\centering
\small
\caption{\textbf{Common phrases in first and second attempts.} Response-level phrase frequencies for Qwen3-1.7B-Base on MATH-500 at training step 1,040. Percentages are computed over nonempty first attempts and nonempty verifier-conditioned retries extracted from one seed's per-step evaluation outputs.}
\vspace{8pt}
\setlength{\tabcolsep}{6pt}
\renewcommand{\arraystretch}{1.08}
\begin{tabular}{@{}lr@{\hspace{1.2cm}}lr@{}}
\toprule
\textbf{First-attempt phrase} & \textbf{Responses} &
\textbf{Second-attempt phrase} & \textbf{Responses} \\
\midrule
``We need to find \ldots'' & 76.3\% & ``Let's approach this problem differently'' & 75.8\% \\
``Step by step'' & 58.4\% & ``We need to find \ldots'' & 74.9\% \\
``We are given \ldots'' & 43.8\% & ``Use the fact that \ldots'' & 37.9\% \\
``The final answer is \ldots'' & 34.2\% & ``Let's consider \ldots'' & 16.8\% \\
``Solve the problem'' & 34.1\% & ``We need to check \ldots'' & 15.5\% \\
``Understand the problem'' & 32.7\% & ``Use a simpler approach'' & 13.8\% \\
``Use the fact that \ldots'' & 16.6\% & ``Use a different method'' & 13.0\% \\
``We will follow these steps'' & 12.8\% & ``We need to re-evaluate \ldots'' & 9.9\% \\
& & ``Ensure we get the correct answer'' & 8.6\% \\
\bottomrule
\end{tabular}
\label{tab:attempt-phrases}
\end{table*}

\begin{table*}[t]
\centering
\small
\caption{\textbf{Common phrases in first and second attempts.} Response-level phrase frequencies for Qwen3-4B-Base on MATH-500 at training step 500. Percentages are computed over nonempty first attempts and nonempty verifier-conditioned retries extracted from one seed's per-step evaluation outputs.}
\vspace{8pt}
\setlength{\tabcolsep}{6pt}
\renewcommand{\arraystretch}{1.08}
\begin{tabular}{@{}lr@{\hspace{1.2cm}}lr@{}}
\toprule
\textbf{First-attempt phrase} & \textbf{Responses} &
\textbf{Second-attempt phrase} & \textbf{Responses} \\
\midrule
``Step by step'' & 84.0\% & ``We need to find \ldots'' & 63.6\% \\
``Final answer'' & 77.5\% & ``Let's rethink this'' & 37.1\% \\
``To find the \ldots'' & 68.0\% & ``Let's try a different approach'' & 33.5\% \\
``We need to find \ldots'' & 43.7\% & ``Let's think differently'' & 28.8\% \\
``To determine the \ldots'' & 15.8\% & ``I didn't arrive at the correct final answer'' & 25.7\% \\
``We can use \ldots'' & 10.5\% & ``Let's consider \ldots'' & 22.3\% \\
``Follow these steps'' & 10.1\% & ``Let's think about \ldots'' & 16.7\% \\
``Let's consider \ldots'' & 8.4\% & ``Use the fact that \ldots'' & 12.6\% \\
& & ``I didn't consider \ldots'' & 10.7\% \\
\bottomrule
\end{tabular}
\label{tab:qwen4b-attempt-phrases}
\end{table*}

\begin{figure*}[t]
\centering
\begin{minipage}{0.96\textwidth}
\begin{tcolorbox}[reapoouter]
\footnotesize

\begin{tcolorbox}[
  enhanced,
  boxrule=0.6pt,
  arc=1.5mm,
  colback=reapoBlueBg,
  colframe=reapoRule,
  left=2mm,
  right=2mm,
  top=0.6mm,
  bottom=0.6mm,
  before skip=0pt,
  after skip=0pt,
]
\scriptsize\textbf{Problem}\hspace{1.5mm} Evaluate $(1+2i)6-3i$. Please reason step by step, and put your final answer within \texttt{\detokenize{\boxed{}}}.
\end{tcolorbox}
\par\nointerlineskip\vspace{0.45mm}\noindent
\begin{minipage}[t]{0.065\linewidth}
\centering
\vspace{0pt}\vspace{1mm}
\reapoQwenIcon
\end{minipage}
\hfill
\begin{minipage}[t]{0.92\linewidth}
\vspace{0pt}
\begin{tcolorbox}[
  reapoattempt,
  colback=reapoRedBg,
  colframe=reapoRed!55,
]
\textcolor{reapoPurple}{\textbf{Attempt 1}}

\vspace{0.3mm}
\scriptsize
To evaluate the expression $(1+2i)(6-3i)$, let's proceed step by step.\par
\textbf{Step 1: Apply the Distributive Property (FOIL Method)}\par
$(1+2i)(6-3i)=6-3i+12i-6i^2.$\par
\ldots{} Recall that $i^2=-1$, so $6-3i+12i+6=12+9i.$\par
\textbf{Final answer:} $\boxed{12+9i}$
\hspace{1.5mm}\textcolor{reapoRed}{\textbf{Incorrect} $\boldsymbol{\times}$}
\end{tcolorbox}
\end{minipage}
\par\nointerlineskip\vspace{0.45mm}\noindent\hspace*{0.09\linewidth}%
\begin{minipage}{0.87\linewidth}
\begin{tcolorbox}[
  enhanced,
  boxrule=0.6pt,
  arc=1.5mm,
  colback=reapoGoldBg,
  colframe=reapoGold!55,
  left=2mm,
  right=2mm,
  top=0.6mm,
  bottom=0.6mm,
  before skip=0pt,
  after skip=0pt,
]
\scriptsize
\textcolor{reapoGold}{\textbf{Retry prompt}}\hspace{1.5mm}
\emph{``Your answer is wrong. Reflect on your previous answer and try something else.''}
\end{tcolorbox}
\end{minipage}
\par\nointerlineskip\vspace{0.45mm}\noindent
\begin{minipage}[t]{0.065\linewidth}
\centering
\vspace{0pt}\vspace{1mm}
\reapoQwenIcon
\end{minipage}
\hfill
\begin{minipage}[t]{0.92\linewidth}
\vspace{0pt}
\begin{tcolorbox}[
  reapoattempt,
  colback=reapoGreenBg,
  colframe=reapoGreen!55,
]
\textcolor{reapoPurple}{\textbf{Attempt 2}}

\vspace{0.3mm}
\scriptsize
You're right, I didn't read the problem correctly. The expression is $(1+2i)6-3i$, not $(1+2i)(6-3i)$. Let's try again.\par
\ldots{} Distribute the $6$ and combine the imaginary terms: $6(1+2i)-3i=6+12i-3i=6+9i.$\par
\textbf{Final answer:} $\boxed{6+9i}$
\hspace{1.5mm}\textcolor{reapoGreen}{\textbf{Correct} $\checkmark$}
\end{tcolorbox}
\end{minipage}

\end{tcolorbox}
\end{minipage}
\caption{A verifier-conditioned retry from ReAPO on MATH-500. The retry feedback causes the model to revisit its initial interpretation and recover the correct answer.}
\label{fig:retry-example}
\end{figure*}

\subsection{Ablations on Qwen3-1.7B-Base}
\label{app:qwen17-ablation}

\Cref{fig:qwen17-ablation} and \Cref{tab:qwen17b-seed42-eval} compare GRPO, GRPO with twice the group size, and ReAPO with $K\in\{2,3\}$ on Qwen3-1.7B-Base at maximum response lengths of 4,096 and 8,192 tokens, using a single training seed. Three observations stand out. First, ReAPO matches or exceeds GRPO with twice the group size while generating about $20\%$ to $30\%$ fewer tokens. Second, a longer response budget helps every method, and ReAPO with a 4,096-token budget performs on par with GRPO given 8,192 tokens ($22.1\%$ versus $21.9\%$). Third, increasing the attempt budget to $K=3$ slightly improves the mathematical benchmarks but reduces GPQA-Diamond accuracy (see \Cref{sec:experiments}).

\subsection{Sequential versus Parallel Sampling}
\label{app:sequential-vs-parallel}

To isolate the benefit of serial dependence during inference, we compare sequential retry sampling against parallel independent sampling on AIME 2025 using Qwen3-1.7B-Base, as shown in \Cref{fig:aime25-sequential-vs-parallel}. For a fixed total sampling budget $k$, the parallel baseline draws $k$ independent first attempts, whereas sequential sampling draws $k/2$ trajectories with up to two verifier-conditioned attempts per trajectory. At very small sampling budgets, parallel sampling performs slightly better, suggesting that independent draws can be preferable when only a handful of samples are available. As the budget increases, however, sequential sampling becomes more sample-efficient: in \Cref{fig:aime25-sequential-vs-parallel}, it reaches the success rate attained by parallel Pass@2048 using approximately $1.7\times$ fewer generated responses. This indicates that conditioning later attempts on previous failures improves solution coverage in the higher-budget regime, even though parallel sampling remains competitive at very low budgets.

\begin{figure}[t]
    \centering
    \includegraphics[width=0.5\textwidth]{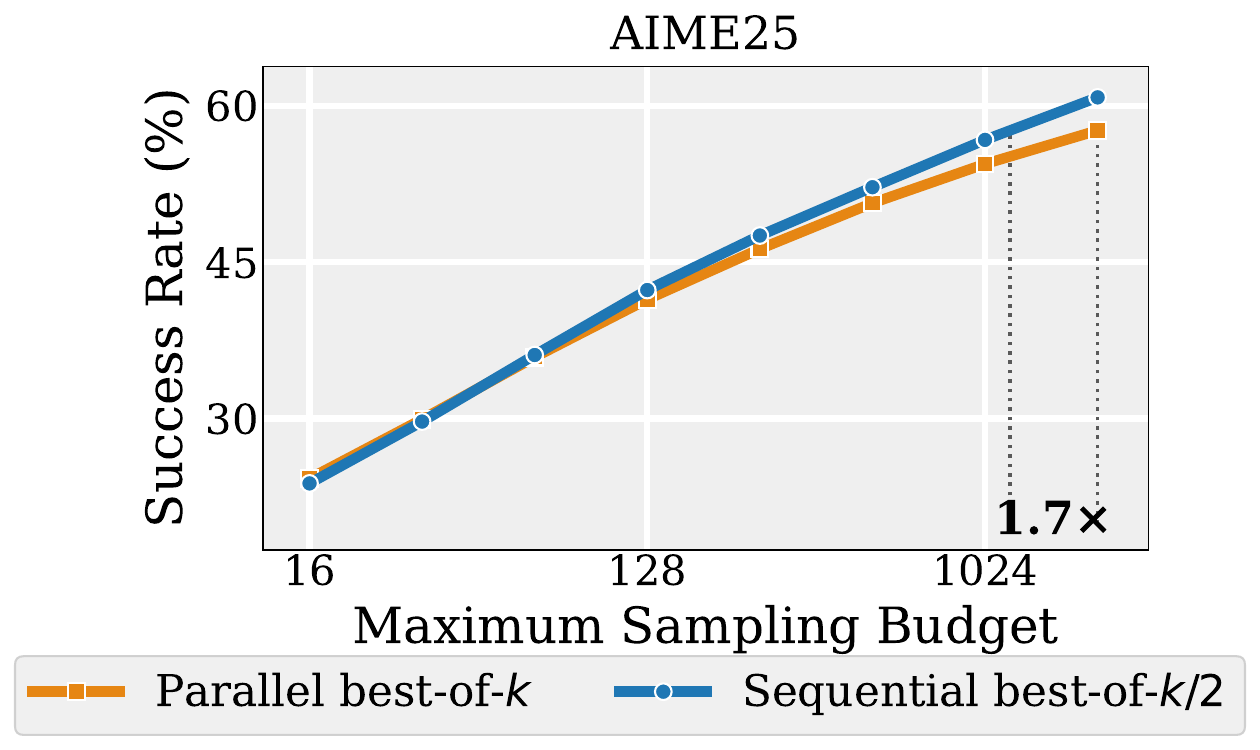}
    \caption{\textbf{Sequential versus parallel sampling on AIME 2025 with Qwen3-1.7B-Base.}
    At total sampling budget $k$, parallel sampling draws $k$ independent first attempts, while sequential sampling draws $k/2$ trajectories with up to two verifier-conditioned attempts each. The dotted lines indicate the sampling budgets required to reach the same success rate, and the annotation reports their ratio.}
    \label{fig:aime25-sequential-vs-parallel}
\end{figure}

\clearpage
\section{Formal Derivation of ReAPO}\label{app:theory}

In this section we provide the formal derivation of the credit assignment mechanism used in ReAPO as well as the proofs of \Cref{prop:policy_gradient_decomposition,prop:time_invariant_policy_gradient}.  Finally, we bound the bias introduced by the truncation of the number of attempts allowed that we instantiate in practice and discuss the theoretical implications of normalizing our advantages by their standard deviation.

To begin we recall the notation of \Cref{sec:method}.  We conduct our analysis entirely with respect to a single fixed prompt $x$ and thus omit the dependence on $x$ in the subsequent derivations and notation.  Recall that the model $\pi_\theta$ generates a sequence of attempts $y_1, y_2, \dots$ associated with a sequence of binary rewards $r_1, r_2, \dots$ until the stopping condition that $r_m = 1$ for some $m$, at which point the interaction terminates.  We define
\begin{align}\label{eq:stopping_time}
    \tau = \inf\left\{ m \geq 1 | r_m = 1 \right\}
\end{align}
to be the stopping time at which the first successful attempt occurs.  We assume that the interaction continues until a successful attempt is made, i.e., $\tau < \infty$ almost surely and moreover that $\ee\left[ \tau \right] < \infty$.  Our primary objective, as motivated by the stochastic shortest paths literature \cite{6dfe7762-0280-326e-88a5-877d93306e7b}, is to minimize the expected number of failed attempts before the first success, i.e.,
\begin{align}
    \cL(\pi_\theta) = \ee_{\pi_\theta}\left[ \tau - 1 \right].
\end{align}
In order to conduct our analysis, we let
\begin{align}\label{eq:q_m}
    q_m(\pi_\theta) = \pp_{\pi_\theta}(r_m = 1 | r_1 = \cdots = r_{m-1} = 0),
\end{align}
i.e. the probability that the $m$-th attempt is successful given that all previous attempts have failed.  We make the simplifying assumption that $q_m(\pi_\theta)$ depends on the history only through the fact that all previous attempts have failed and does not depend on the specific outcomes of those previous attempts.

We also introduce
\begin{align}\label{eq:alpha_m}
    \alpha_m(\pi_\theta) = \pp_{\pi_\theta}(\tau \geq m) = \prod_{j = 1}^{m-1} \left( 1 - q_j(\pi_\theta) \right),
\end{align}
the probability that the interaction has not terminated before the $m$-th attempt.  Thus, while $q_m(\pi_\theta)$ may not be monotone in $m$, $\alpha_m$ is always decreasing.  Finally, we define
\begin{align}\label{eq:c_m}
    c_m(\pi_\theta) = \sum_{k = m}^\infty \prod_{j = m+1}^k \left( 1 - q_j(\pi_\theta) \right),
\end{align}
the expected number of additional attempts needed starting from attempt $m+1$, given that the first $m$ attempts have failed.  Here the empty product in the $k=m$ term equals $1$, accounting for attempt $m+1$.

\subsection{Derivation of the Policy Gradient Decomposition}

We now restate and prove \Cref{prop:policy_gradient_decomposition}.
\begin{proposition}\label{prop:policy_gradient_decomposition_formal}
    Suppose that for all $\theta$ in a sufficiently small neighborhood of $\theta_0$, it holds that $\ee_{\pi_\theta}[\tau] < \infty$.  Suppose further that for all $m \geq 1$, it holds that
    \begin{align}
        \pp_{\pi_\theta}(r_m = 1 | x, y_1, \dots, y_{m-1}) = q_m(\pi_\theta),
    \end{align}
    i.e. the probability that the $m$-th attempt is successful depends only on the fact that all previous attempts have failed and not on the specific outcomes of those attempts.  Let $\alpha_m$ and $c_m$ be defined as in \eqref{eq:alpha_m} and \eqref{eq:c_m}, respectively.  Moreover suppose that
    \begin{align}\label{eq:regularity_conditions}
        q_m(\theta) \geq q_{\min} > 0 \quad \text{and} \quad \norm{\nabla_\theta q_m(\theta)} \leq B.
    \end{align}
    Then it holds that
    \begin{align}
        -\nabla_\theta \cL(\pi_\theta) = \sum_{m = 1}^\infty \alpha_m(\pi_\theta) \nabla_\theta q_m(\pi_\theta) c_m(\pi_\theta).
    \end{align}
\end{proposition}
\begin{proof}
    We first observe that by the tail sum formula for expectations,
    \begin{align}
        \cL(\pi_\theta) = \sum_{m = 2}^\infty \pp_{\pi_\theta}(\tau \geq m) = \sum_{m = 2}^\infty \alpha_m(\pi_\theta) = \sum_{k = 1}^\infty \alpha_{k+1}(\pi_\theta).
    \end{align}
    Now, for all $k \geq 1$, it is immediate by definition that
    \begin{align}
        \alpha_{k+1}(\pi_\theta) = \prod_{j = 1}^k (1 - q_j(\pi_\theta)).
    \end{align}
    By \eqref{eq:regularity_conditions}, we can differentiate under the infinite sum.  Differentiating, we have
    \begin{align}
        - \nabla_\theta \alpha_{k+1}(\pi_\theta) = \sum_{m= 1}^k \alpha_m(\pi_\theta) \nabla_\theta q_m(\pi_\theta) \prod_{j = m+1}^k (1 - q_j(\pi_\theta)).
    \end{align}
    Summing over $k$ from $1$ to $\infty$, we obtain
    \begin{align}
         - \nabla_\theta \cL(\pi_\theta) &=  \sum_{k = 1}^\infty - \nabla_\theta \alpha_{k+1}(\pi_\theta) \\
         &= \sum_{k = 1}^\infty \sum_{m= 1}^{k} \alpha_m(\pi_\theta) \nabla_\theta q_m(\pi_\theta) \prod_{j = m+1}^k (1 - q_j(\pi_\theta)) \\
         &= \sum_{m = 1}^\infty \alpha_m(\pi_\theta) \nabla_\theta q_m(\pi_\theta) \sum_{k = m}^\infty \prod_{j = m+1}^k (1 - q_j(\pi_\theta)) \\
         &= \sum_{m = 1}^\infty \alpha_m(\pi_\theta) \nabla_\theta q_m(\pi_\theta) c_m(\pi_\theta).
    \end{align}
    The result holds.
\end{proof}
The inclusion of $c_m(\pi_\theta)$ is inconvenient due to its dependence on a quantity that is hard to estimate.  We thus make a stronger assumption, namely that $c_m(\pi_\theta)$ is \emph{independent of } $m$ and denote this common value by $c(\pi_\theta)$.  This assumption significantly simplifies the policy gradient expression.  We now restate and prove \Cref{prop:time_invariant_policy_gradient}.
\begin{proposition}\label{prop:time_invariant_policy_gradient_formal}
    Under the assumptions of \Cref{prop:policy_gradient_decomposition_formal}, suppose that $c_m(\pi_\theta)$ is independent of $m$ with common value $c(\pi_\theta)$. Define
    \begin{align}
        \nabla_\theta \lreapo(\pi_\theta):
        = \ee_{\pi_\theta}\left[ \frac 1G \sum_{g=1}^G \sum_{k=1}^{\tau_g}
        r_{g,k} \cdot \nabla_\theta \log \pi_\theta(y_{g,k} \mid h_{g,k-1}) \right].
    \end{align}
    Then it holds that
    \begin{align}
        -\nabla_\theta \cL(\pi_\theta)
        = c(\pi_\theta) \sum_{m=1}^\infty \alpha_m(\pi_\theta) \nabla_\theta q_m(\pi_\theta)
        = c(\pi_\theta) \cdot \nabla_\theta \lreapo(\pi_\theta).
    \end{align}
\end{proposition}
\begin{proof}
    The first equality follows immediately from \Cref{prop:policy_gradient_decomposition_formal} and the assumption that $c_m(\pi_\theta)$ is independent of $m$. For the second equality, fix a history $h=h_{g,k-1}$ on which attempt $k$ is reached. The history-independence assumption holds in a neighborhood of the parameter being differentiated, so $\sum_y \pi_\theta(y \mid h)r(x,y) = q_k(\pi_\theta)$. Holding $h$ fixed and differentiating gives
    \begin{align}
        \ee_{\pi_\theta}\left[ r_{g,k} \nabla_\theta \log \pi_\theta(y_{g,k} \mid h) \mid h_{g,k-1}=h \right]
        &= \sum_y r(x,y) \nabla_\theta \pi_\theta(y \mid h) = \nabla_\theta q_k(\pi_\theta).
    \end{align}
    Finally, attempt $k$ is reached with probability $\alpha_k(\pi_\theta)$. Averaging over reached histories and summing over attempts gives
    \begin{align}
        \nabla_\theta \lreapo(\pi_\theta)
        = \sum_{k=1}^\infty \alpha_k(\pi_\theta) \nabla_\theta q_k(\pi_\theta),
    \end{align}
    which proves the second equality.
\end{proof}

To illustrate the conclusion of \Cref{prop:time_invariant_policy_gradient_formal}, consider the simple setting where the first attempt succeeds with probability $p_\theta = q_1(\pi_\theta)$ and for $m > 1$, the success probability is $q_m(\pi_\theta) = q_\theta$. In words, after the first attempt, each attempt succeeds with a constant probability. Then it is immediate that $c_m(\pi_\theta) = q_\theta^{-1}$ for all $m \geq 1$, and hence
\begin{align}
    -\nabla_\theta \cL(\pi_\theta) = \frac{\nabla_\theta \lreapo(\pi_\theta)}{q_\theta}.
\end{align}

\subsection{The Effect of Truncation}

While the above theory was derived assuming an infinite number of attempts, in practice, we often truncate the number of attempts to a finite maximum $K$.  This truncation introduces a bias in the policy gradient estimate, as it ignores the contributions from attempts beyond the $K$-th one.  To better understand the effect of the bias, consider the truncated objective
\begin{align}
    \cL_K(\pi_\theta) = \ee_{\pi_\theta}\left[ (\tau - 1) \wedge K \right],
\end{align}
i.e. the expected number of failed attempts, truncated at $K$.  We first observe that as $K \uparrow \infty$, the truncated objective $\cL_K(\pi_\theta)$ converges to the original objective $\cL(\pi_\theta)$ at an exponential rate under a mild condition on $q_m(\pi_\theta)$ for $m > 1$.
\begin{proposition}
    For any choice of $K$ and any $\pi_\theta$, it holds that
    \begin{align}
        \cL(\pi_\theta) - \cL_K(\pi_\theta) = \ee_{\pi_\theta}\left[ \left( \tau - 1 - K \right)_+ \right] = \sum_{m = K + 1}^\infty \alpha_{m+1}(\pi_\theta).
    \end{align}
    Moreover, if for all $m \geq 1$ it holds that $q_m(\pi_\theta) \geq q_{\min} > 0$ uniformly in $m$, then
    \begin{align}
        \cL(\pi_\theta) - \cL_K(\pi_\theta) \leq \frac{(1 - q_{\min})^{K+1}}{q_{\min}}.
    \end{align}
\end{proposition}
\begin{proof}
    The first statement is immediate from the tail-sum formula.  For the second statement, observe that if $q_m(\pi_\theta) \geq q_{\min}$ then
    \begin{align}
        \alpha_{m+1} = \prod_{k = 1}^m (1 - q_k(\pi_\theta)) \leq (1 - q_{\min})^m.
    \end{align}
    Summing over $m = K + 1, K + 2, \dots$ gives the result.
\end{proof}
Note that a consequence of the previous result suggests that as long as $K \gtrsim -\frac{\log(q_{\min} \epsilon)}{q_{\min}}$, it holds that $\cL_{K}(\pi_\theta) \approx \cL(\pi_\theta)$ up to an error of at most $\epsilon$.
Differentiating the tail-sum representation above, we obtain
\begin{align}
    \nabla_\theta \cL(\pi_\theta) - \nabla_\theta \cL_K(\pi_\theta)
    &= \sum_{m=K+1}^\infty \nabla_\theta \alpha_{m+1}(\pi_\theta) \\
    &= -\sum_{m=K+1}^\infty \sum_{j=1}^m
    \alpha_j(\pi_\theta) \nabla_\theta q_j(\pi_\theta)
    \prod_{\ell=j+1}^m (1-q_\ell(\pi_\theta)).
\end{align}
In particular, if $q_m(\pi_\theta) \geq q_{\min} > 0$ and
$\norm{\nabla_\theta q_m(\pi_\theta)} \leq B$ uniformly in $m$,
then
\begin{align}
    &\norm{\nabla_\theta \cL(\pi_\theta) - \nabla_\theta \cL_K(\pi_\theta)} \\
    &\qquad \leq B \sum_{m=K+1}^\infty m(1-q_{\min})^{m-1} \\
    &\qquad = B\left(
        \frac{(K+1)(1-q_{\min})^K}{q_{\min}}
        + \frac{(1-q_{\min})^{K+1}}{q_{\min}^2}
    \right).
\end{align}
Thus, the bias in the policy gradient due to truncation decays
exponentially with $K$.

\subsection{The Effect of Advantage Normalization}
While \Cref{prop:time_invariant_policy_gradient_formal} provides a clear motivation for ReAPO, we note that the loss does not use the GRPO baseline we implement in practice.  The following corollary formalizes the connection between the ReAPO objective and the time-invariant policy gradient form.
\begin{corollary}
    Under the same assumptions as \Cref{prop:time_invariant_policy_gradient_formal}, let $\tau_g$ denote the number of attempts for group $g$ and let
    \begin{align}
    A_{g,k}
    =
    r_{g,k}
    -
    \frac{
    \sum_{g=1}^G\sum_{k=1}^{\tau_g} r_{g,k}
    }{
    \sum_{g=1}^G \tau_g
    }.
    \end{align}
    Suppose that $K = \infty$, and define
    \begin{align}
    \nabla_\theta \cL_{\mathrm{ReAPO-center}}(\pi_\theta)
    :=
    \ee_{\pi_\theta}\left[
    \frac{1}{G}
    \sum_{g=1}^G
    \sum_{k=1}^{\tau_g}
    A_{g,k}
    \nabla_\theta
    \log \pi_\theta(y_{g,k}\mid h_{g,k-1})
    \right].
    \end{align}
    Then, writing $\tau = \sum_{g=1}^G \tau_g$,
    \begin{align}
    \nabla_\theta \cL_{\mathrm{ReAPO-center}}(\pi_\theta)
    =
    \nabla_\theta \cL_{\mathrm{ReAPO}}(\pi_\theta)
    -
    \nabla_\theta
    \ee_{\pi_\theta}\left[\frac{1}{\tau}\right].
    \end{align}
    Consequently,
    \begin{align}
    -\nabla_\theta L(\pi_\theta)
    =
    c(\pi_\theta)
    \left(
    \nabla_\theta L_{\mathrm{center}}(\pi_\theta)
    +
    \nabla_\theta
    \ee_{\pi_\theta}\left[\frac{1}{\tau}\right]
    \right).
    \end{align}
\end{corollary}
\begin{proof}
    First note that because $\tau_g < \infty$ almost surely and $r_{g,k} = 0$ for $k < \tau_g$, it holds that
    \begin{align}
        \sum_{g=  1}^G \sum_{k = 1}^{\tau_g} r_{g,k} = G.
    \end{align}
    Thus,
    \begin{align}
        A_{g,k} = r_{g,k} - \frac{G}{\tau},
    \end{align}
    and the result follows from \Cref{prop:time_invariant_policy_gradient_formal} and the linearity of the gradient.
\end{proof}
Note that the centering by the mean, as done in $\cL_{\mathrm{ReAPO-center}}$ thus does not exactly correspond to minimizing $\ee_{\pi_\theta}[\tau - 1]$, but that the correction term $\nabla_\theta \ee_{\pi_\theta}\left[\frac{1}{\tau}\right]$ accounts for this discrepancy and also serves to encourage $\tau$ to be as small as possible.

While all of the above theory was developed in the context of unnormalized advantages, i.e. using rewards directly without centring or dividing by the standard deviation, in practice we use advantage normalization, dividing by the standard deviation of the rewards in addition to centring by the mean, as is done in GRPO \citep{shao2024deepseekmathpushinglimitsmathematical}.  Following Dr.GRPO \cite{liu2025understandingr1zeroliketrainingcritical}, we can analyze the effect of this normalization on the bias and variance of the policy gradient estimator.  In the setting of interest where rewards are binary and the probability of a given attempt being correct is $\mu_K$, then it holds that the standard deviation is $\sigma_K = \sqrt{\mu_K (1 - \mu_K)}$.  In the limit as $K \uparrow \infty$, it holds that
\begin{align}
    \mu_\infty = \frac{G}{\sum_{g = 1}^G \tau_g}.
\end{align}
Thus, as $G \to \infty$, it holds that $\mu_\infty \to \nicefrac{1}{\ee_\theta[\tau]}$.
Thus for each prompt $x$ in the limit as $K,G \uparrow \infty$, the standard deviation of the rewards is
\begin{align}
    \sigma_\infty(x) = \frac{\sqrt{\ee\left[ \tau(x) - 1 \right]}}{\ee\left[ \tau(x) \right]},
\end{align}
where we have now restored the dependence on the prompt to allow for considerations of multiple prompts and their respective attempt distributions.  Thus the standard deviation normalization has no effect on the direction of the gradient, but does lead to a rescaling of the gradient magnitude, about on the order of $\sqrt{\ee[\tau]}$ for each prompt $x$.  In particular, more difficult problems with larger expected numbers of attempts will have their policy gradient contributions up-weighted relative to easier problems with smaller expected numbers of attempts.

\end{document}